\documentclass[11pt]{article}

\usepackage{url}
\usepackage{amssymb,enumerate}
\usepackage{amsmath,amsfonts}
\usepackage{amsthm}
\usepackage{latexsym}
\usepackage{mathrsfs}
\usepackage{color}
\usepackage{microtype}
\usepackage{geometry}

\usepackage{algorithm}
\usepackage{algorithmicx}
\usepackage{algpseudocode}

\usepackage{multirow}
\allowdisplaybreaks[4]

\usepackage{amsmath,amsfonts,bm}

\def\1{\bm{1}}

\DeclareMathAlphabet{\mathsfit}{\encodingdefault}{\sfdefault}{m}{sl}
\SetMathAlphabet{\mathsfit}{bold}{\encodingdefault}{\sfdefault}{bx}{n}

\newcommand{\E}{\mathbb{E}}

\newcommand{\R}{\mathbb{R}}

\DeclareMathOperator*{\argmin}{arg\,min}

\theoremstyle{plain}
\newtheorem{theorem}{Theorem}

\newtheorem{lemma}{Lemma}

\newtheorem{assumption}{Assumption}
\newtheorem{proposition}{Proposition}

\newtheorem{definition}{Definition}

\newtheorem{remark}{Remark}

\usepackage{bbm}
\usepackage{graphicx, color}
\usepackage{algorithm, algpseudocode} 
\usepackage{mathrsfs} 

\usepackage{lipsum}

\title{Bias-Variance Trade-off for Clipped Stochastic First-Order Methods:\\ From Bounded Variance to Infinite Mean}
\author{Chuan He\thanks{Department of Mathematics, Link\"oping University, Sweden (email: {\tt chuan.he@liu.se}). This work was partially supported by the Wallenberg AI, Autonomous Systems and Software Program (WASP) funded by the Knut and Alice Wallenberg Foundation.}
}

\date{
\today
}

\begin{document}

\maketitle
	
\begin{abstract}
Stochastic optimization is fundamental to modern machine learning. Recent research has extended the study of stochastic first-order methods (SFOMs) from light-tailed to heavy-tailed noise, which frequently arises in practice, with clipping emerging as a key technique for controlling heavy-tailed gradients. Extensive theoretical advances have further shown that the oracle complexity of SFOMs depends on the tail index $\alpha$ of the noise. Nonetheless, existing complexity results often cover only the case $\alpha \in (1,2]$, that is, the regime where the noise has a finite mean, while the complexity bounds tend to infinity as $\alpha$ approaches $1$. This paper tackles the general case of noise with tail index $\alpha\in(0,2]$, covering regimes ranging from noise with bounded variance to noise with an infinite mean, where the latter case has been scarcely studied. Through a novel analysis of the bias-variance trade-off in gradient clipping, we show that when a symmetry measure of the noise tail is controlled, clipped SFOMs achieve improved complexity guarantees in the presence of heavy-tailed noise for any tail index $\alpha \in (0,2]$. Our analysis of the bias-variance trade-off not only yields new unified complexity guarantees for clipped SFOMs across this full range of tail indices, but is also straightforward to apply and can be combined with classical analyses under light-tailed noise to establish oracle complexity guarantees under heavy-tailed noise. Finally, numerical experiments validate our theoretical findings.
\end{abstract}


\medskip


\section{Introduction}

Stochastic first-order methods (SFOMs) have been instrumental in driving recent progress in machine learning. In contrast to deterministic first-order methods, SFOMs cannot directly access exact gradients and instead rely on stochastic gradient estimates, with estimation noise originating from various sources such as sampling \cite{bottou2010large} and deliberate injection to improve generalization \cite{he2019parametric, igl2019generalization} or to preserve privacy \cite{abadi2016deep}. Motivated by their widespread applications, studies of SFOMs have garnered considerable attention, advancing the theoretical foundations of optimization and leading to powerful algorithmic frameworks that incorporate stochasticity grounded in statistical principles. In this paper, we consider the composite optimization problem:
\begin{align}\label{ucpb}
\min_{x\in\R^n} \{F(x) := f(x) + h(x)\}.
\end{align}
where $f:\R^n\to\R$ is continuously differentiable and $h:\R^n\to(-\infty,\infty]$ is lower semicontinuous and convex. Problem \eqref{ucpb} has broad applications in machine learning, where $f$ typically represents the loss function, reflecting errors over the training data, and $h$ usually denotes a regularization term that enforces desirable properties on the model. We refer the readers to \cite{sra2011optimization} for further details on the applications of problem \eqref{ucpb}.

There have been extensive algorithmic developments on SFOMs for solving \eqref{ucpb} or its special cases; e.g., see \cite{alacaoglu2025towards,bottou2018optimization,davis2021low,foster2019complexity,gao2024non,ghadimi2012optimal,ghadimi2013optimal,lan2012optimal,liang2024single,moulines2011non,nemirovski2009robust,nemirovski1983problem,polyak1992acceleration,robbins1951stochastic,shalev2009stochastic}. In the classical setting of SFOMs, the exact gradient $\nabla f$ is unavailable. Instead, we have access to a stochastic oracle $G:\R^n\times\Xi\to\R^n$ that satisfies the unbiasedness and bounded-variance conditions:
\begin{align}\label{asp:unbias-variance}
    \E[G(x;\xi)] = \nabla f(x),\quad \E[\|G(x;\xi) - \nabla f(x)\|^2]\le\sigma^2\qquad\forall x\in\R^n
\end{align}
for some $\sigma>0$. Under this condition and the assumption that $\nabla f$ is Lipschitz continuous, many SFOMs have been proposed for solving \eqref{ucpb} and its special cases. In particular, when $f$ is convex and $h$ is the indicator function of a simple closed convex set, an accelerated stochastic proximal gradient method (SPGM) has been developed in \cite{lan2012optimal} to achieve a first-order oracle complexity of $\mathcal{O}(\epsilon^{-2})$ for finding an $\epsilon$-stochastic optimal solution (see Definition \ref{def:sol} for its precise definition). When $f$ is strongly convex, accelerated SPGMs have been developed in \cite{ghadimi2012optimal,ghadimi2013optimal} to achieve a first-order oracle complexity of $\mathcal{O}(\epsilon^{-1})$ for finding an $\epsilon$-stochastic optimal solution. Moreover, when $f$ is generally nonconvex, an SPGM with momentum have been developed in \cite{gao2024non} to achieve a first-order oracle complexity of $\mathcal{O}(\epsilon^{-4})$ for finding an $\epsilon$-stochastic stationary point (see Definition \ref{def:sol} for its precise definition).

Recently, there has been growing interest in stochastic optimization under heavy-tailed noise, extending beyond condition \eqref{asp:unbias-variance}---a trend largely driven by modern applications such as transformer training \cite{zhang2020adaptive} and advanced privacy-preserving techniques \cite{csimcsekli2024privacy}. In particular, to model heavy-tailed noise, one may extend \eqref{asp:unbias-variance} to impose unbiasedness and a finite $\alpha$th central moment, for some $\alpha \in (1,2]$, on the stochastic gradient $G$; that is, one assumes that
\begin{align}\label{asp:heavy-tail}
\E[G(x;\xi)] = \nabla f(x),\quad \E[\|G(x;\xi) - \nabla f(x)\|^\alpha]\le\sigma^\alpha\qquad\forall x\in\R^n   
\end{align}
for some $\sigma>0$. As an algorithmic strategy for handling noise, clipping has been widely incorporated into modern deep learning \cite{pascanu2013difficulty}, and extensive studies have investigated clipped SFOMs for solving problem \eqref{ucpb} or its special cases under the heavy-tailed condition in \eqref{asp:heavy-tail}; see, e.g., \cite{cutkosky2021high, gorbunov2020stochastic, liu2024high, nguyen2023improved, sadiev2023high, sadiev2025second, zhang2020adaptive}. In particular, clipped SFOMs have been analyzed in \cite{zhang2020adaptive} for solving the unconstrained special case of problem \eqref{ucpb} with $h=0$. When $f$ is strongly convex, a first-order oracle complexity of $\mathcal{O}(\epsilon^{-\alpha/(2(\alpha-1))})$ has been established for finding an $\epsilon$-stochastic optimal solution. In the generally nonconvex case, an oracle complexity of $\mathcal{O}(\epsilon^{-(3\alpha-2)/(\alpha-1)})$ has been established for finding an $\epsilon$-stochastic stationary point (i.e., a point whose gradient norm is at most $\epsilon$ in expectation). It has also been shown in \cite{zhang2020adaptive} that both complexity bounds match the corresponding lower bounds. In addition, it has been shown in \cite{sadiev2023high} that when $f$ is convex, the clipped SFOM achieves an oracle complexity of $\mathcal{O}(\epsilon^{-\alpha/(\alpha-1)})$ for finding an $\epsilon$-stochastic optimal solution, which matches the lower bound in \cite{nemirovski1979efficient}. More recently, it has been pointed out in \cite{fatkhullin2025can} that the vanilla projected SGD also achieves these optimal complexity bounds for solving the special case of \eqref{ucpb} with $h$ being an indicator function of a simple closed convex set. Furthermore, in \cite{he2025accelerated}, it has been show that the vanilla accelerated SPGM achieves the oracle complexity that is universally optimal for smooth, weakly smooth, and nonsmooth convex optimization, as well as stochastic convex optimization under the heavy-tailed condition \eqref{asp:heavy-tail}. For nonconvex problems, normalized SFOMs can also attain the aforementioned optimal complexity, indicating that normalization appears to be a viable alternative to clipping for handling heavy-tailed noise; e.g., see \cite{he2025complexity,hubler2024gradient,liu2025nonconvex,sun2025revisiting}. These results suggest that studies of clipped SFOMs under condition \eqref{asp:heavy-tail} may not fully justify the advantages of gradient clipping for handling heavy-tailed noise, as vanilla stochastic algorithms and (potentially simpler) alternative algorithms can achieve comparable complexity bounds. On another note, all these complexity bounds tend to infinity as $\alpha\to1$, failing to cover cases where the noise follows heavy-tailed distributions with an infinite mean, such as the Cauchy and Lévy distributions. This naturally raises the question:
\begin{center}
    \textit{Can we develop and analyze SFOMs under heavy-tailed noise with a possibly infinite mean?}
\end{center}

This paper provides an affirmative answer to this question. We show that clipped SPGMs can provably solve problem \eqref{ucpb} in the presence of heavy-tailed noise with a potentially infinite mean, provided that a symmetry measure of the noise tail is appropriately controlled. Specifically, we first introduce heavy-tailed conditions for any tail index $\alpha \in (0,2]$, characterized by the bounded central moment as in \eqref{alpha-moment} and the power-law condition of the density as in \eqref{density-tail}. We then provide the regularity conditions: the noise distributions are asymptotically unbiased as in \eqref{truncate-bias}, and a measure of symmetry for the noise tail is controlled as in \eqref{near-sym}. Conditions  \eqref{truncate-bias} and \eqref{near-sym} include heavy-tailed distributions that are symmetric about the origin, such as the standard Cauchy distribution, and are much weaker than imposing perfectly symmetry. Since the expectation of the stochastic gradient with infinite-mean noise is undefined, it is natural to consider using clipped stochastic gradients, which have finite first- and second-order moments but also introduce bias (e.g., see \cite{koloskova2023revisiting}). To facilitate the complexity analysis of clipped SPGMs, we further investigate the bias and variance of clipped stochastic gradients, which display a clear trade-off pattern. As a result, our analysis indicates that a moderate clipping threshold is required to efficiently obtain a solution of desirable accuracy, avoiding excessive variance from a threshold that is too large or excessive bias from a threshold that is too small. Finally, we leverage the bias-variance trade-off to establish the first-order oracle complexity of a clipped SPGM (Algorithm \ref{alg:c-spgm}) for solving problem \eqref{ucpb} when $f$ is convex, and of a clipped SPGM with momentum when $f$ is nonconvex (Algorithm \ref{alg:c-spgm-m}). Our oracle complexity bounds are summarized in Table \ref{table:complexity} for ease of reference.


\begin{table}[h!]
\centering
\caption{Oracle Complexity Bounds for Clipped SPGMs under Different Noise Conditions}
\label{table:complexity}
\smallskip
\resizebox{\textwidth}{!}{
\begin{tabular}{c|c|c|c|c}
\hline
& \multirow{2}{*}{solution type} & light-tailed & $\alpha$-heavy-tailed & $\alpha$-symmetric \& $\alpha$-heavy-tailed  \\
& &  ($\alpha=2$)  & ($\alpha\in(1,2)$) & ($\alpha\in(0,2)$ under Assumption \ref{asp:decay})  \\[4pt]
\hline
& & & & \\[-9pt]
strongly convex & \multirow{2}{*}{$\E[F(x)-F^*]\le\epsilon$} &$\widetilde{\mathcal{O}}(\epsilon^{-1})$ & $\widetilde{\mathcal{O}}(\epsilon^{-\alpha/(2(\alpha-1))})$ &  $\widetilde{\mathcal{O}}(\epsilon^{-(\alpha+2)/(2\alpha)})$ \\[2pt]
convex & & $\mathcal{O}(\epsilon^{-2})$ & $\mathcal{O}(\epsilon^{-\alpha/(\alpha-1)})$ &  $\mathcal{O}(\epsilon^{-(\alpha+2)/\alpha})$ \\[4pt]
\hline
& & & & \\[-9pt]
nonconvex & $\E[\mathrm{dist}^2(0,\partial F(x))]\le\epsilon^2$ &  $\mathcal{O}(\epsilon^{-4})$ & $\mathcal{O}(\epsilon^{-(3\alpha-2)/(\alpha-1)})$ & $\mathcal{O}(\epsilon^{-(3\alpha+2)/\alpha})$ \\[4pt]
\hline
\end{tabular}
}
\end{table}

Our main contributions are twofold.
\begin{itemize}
    \item We provide novel unified oracle complexity guarantees for clipped SPGMs under heavy-tailed noise with any tail index $\alpha \in (0,2]$, provided that a symmetry measure of the noise tail is appropriately controlled.
    \item We establish the bias-variance trade-off for clipped stochastic gradients, which can be combined with the analysis of SFOMs under light-tailed noise. As a result, we derive, to the best of our knowledge, the first complexity guarantees for clipped SPGMs for solving convex and nonconvex composite optimization problems.
\end{itemize}

It is noteworthy that some other works also study clipped SFOMs under (nearly) symmetric noise; see, e.g., \cite{armacki2025large, chen2020understanding, jakovetic2023nonlinear}. These works often assume that the noise density is positive near zero and that it is either perfectly symmetric or a mixture of symmetric and asymmetric densities. This is fundamentally different from our setup: we allow the density to be void near zero and require only that the noise tail be nearly symmetric at a specific controlling rate (see Assumption \ref{asp:decay} below). Moreover, in \cite{puchkin2024breaking}, SFOMs equipped with clipping with median-of-means gradient estimators have been developed and analyzed under a mixture of symmetric and asymmetric noise. However, their assumption on noise symmetry requires $k$-fold convolution, which is more complicated than our setup.

The rest of this paper is organized as follows. In Section \ref{sec:not}, we introduce the notation and assumptions used throughout this paper. In Section \ref{sec:bias-var}, we establish the bias-variance trade-off for clipping. In Section~\ref{sec:complexity}, we establish complexity bounds for a clipped SPGM for solving convex problems and a clipped SPGM with momentum for solving nonconvex problems. Section \ref{sec:nr} presents numerical results. In Section \ref{sec:pf}, we provide the proofs of our main results.

\section{Notation and Assumptions}\label{sec:not}
Throughout this paper, we use $\R^n$ to denote the $n$-dimensional Euclidean space. We denote the Euclidean norm and $\ell_\infty$-norm for vectors by $\|\cdot\|$ and $\|\cdot\|_{\infty}$, respectively. For any set $\mathcal{S}\subseteq\R^m$ and $v\in\R^m$, we denote the Euclidean projection onto $\mathcal{S}$ as $\Pi_{\mathcal{S}}(v)$. For any proper closed convex function $\varphi$, we denote its subdifferential by $\partial \varphi$ and define the proximal mapping associated with $\varphi$, with parameter $\eta>0$, as $\mathrm{prox}_{\eta\varphi}(x) := \argmin_{z\in\R^n}\{\varphi(z) + (2\eta)^{-1}\|z-x\|^2\}$. We denote the domain of $\varphi$ as $\mathrm{dom}\,\varphi$. We let $G=[G_1\ \cdots\ G_n]^T:\R^n\times\Xi\to\R^n$ be a stochastic estimator for $\nabla f=[\nabla_1 f\ \cdots\ \nabla_n f]^T$. For every $x\in\R^n$ and $\xi\in\Xi$, we define the stochastic gradient with coordinate-wise clipping as
$G_\tau(x;\xi)=\Pi_{\{g:\|g\|_\infty\le\tau\}}(G(x;\xi))$, and define the estimation noise as $N(x;\xi):=G(x;\xi) - \nabla f(x)$, which has the coordinate representation:
\begin{equation}\label{grad-noise}
N(x;\xi) = \left[\begin{matrix}
N_1(x;\xi)\\
\vdots\\
N_n(x;\xi)
\end{matrix}\right]\qquad\forall x\in\R^n,\xi\in\Xi.  
\end{equation}
For each $1\le i\le n$ and $x\in\R^n$, we let $p_{i,x}$ be the density function of the random variable $N_i(x;\xi)$.  For any $s\in\R$, we let $\mathrm{sgn}(s)$ be $1$ if $s\ge0$ and let it
be $-1$ otherwise. In addition, we use $\mathcal{O}(\cdot)$ to denote the standard big-O notation, $\widetilde{\mathcal{O}}(\cdot)$ to denote big-O notation with hidden logarithmic factors, and $o(\cdot)$ to denote the standard small-o notation.

We now make our main assumptions throughout this paper.

\begin{assumption}\label{asp:basic}
\begin{enumerate}
\item[{(a)}] The proximal operator associated with $h$ can be evaluated exactly, and its domain $\mathrm{dom}\,h$ is bounded.
\item[{(b)}] The gradient $\nabla f$ is Lipschitz continuous on $\mathrm{dom}\,h$, i.e., there exists $L_f>0$ such that $\|\nabla f(y) - \nabla f(x)\|\le L_f\|y - x\|$ for all $x,y\in\mathrm{dom}\,h$.
\item[{(c)}] There exist $\alpha\in(0,2]$ and $\Lambda_1,\Lambda_2,u_1>0$ such that the density functions $\{p_{i,x}\}$ satisfy
\begin{subequations}
\begin{align}
\sup_{\substack{x \in \mathrm{dom}\,h, \\ 1 \le i \le n}}\Big\{\int_{-\infty}^\infty |u|^\alpha p_{i,x}(u) \mathrm{d}u\Big\} & \le \Lambda_1,\qquad \label{alpha-moment}\\
\sup_{\substack{x \in \mathrm{dom}\,h, \\ 1 \le i \le n}}\{p_{i,x}(u)\}& \le \Lambda_2 |u|^{-(\alpha+1)}\qquad\ \forall |u|\ge u_1,\label{density-tail}\\
\lim_{\tau\to\infty} \sup_{\substack{x \in \mathrm{dom}\,h, \\ 1 \le i \le n}}\Big\{\Big| \int^\tau_{-\tau} u p_{i,x}(u)\mathrm{d}u \Big|\Big\} & = 0,\label{truncate-bias}\\
\lim_{\tau\to\infty} \sup_{\substack{x \in \mathrm{dom}\,h, \\ 1 \le i \le n}}\Big\{\Big|\tau\int_\tau^{\infty} (p_{i,x}(u)-p_{i,x}(-u))\mathrm{d}u\Big|\Big\}&=0.\label{near-sym}
\end{align}    
\end{subequations}
\end{enumerate}
\end{assumption}

\begin{remark}
(i) Assumption \ref{asp:basic}(a) is quite common in stochastic optimization. We define the diameter of $\mathrm{dom}\,h$, the upper bound for $\nabla f$ over $\mathrm{dom}\,h$, and the lower bound of $F$ over $\mathrm{dom}\,h$, respectively, as follows:
\begin{equation}\label{def:Uf-Dh}
D_h = \sup_{x,y\in\mathrm{dom}\,h}\{\|x-y\|\},\quad U_f = \sup_{x\in\mathrm{dom}\,h}\{\|\nabla f(x)\|_\infty\},\quad F_{\mathrm{low}} = \inf_{x\in\mathrm{dom}\,h}\{F(x)\}.  
\end{equation}
For convenience, we denote a lower bound for the clipping threshold as
\begin{align}\label{def:tau-underline}
\tau_{(1)} = u_1 + U_f.
\end{align}
Assumption \ref{asp:basic}(b) is standard. It implies the following descent inequality:
\begin{equation}\label{ineq:descent}
f(y) \le f(x) + \nabla f(x)^T(y-x) + \frac{L_f}{2}\|y-x\|^2\qquad\forall x,y\in\mathrm{dom}\,h.
\end{equation}

(ii) Assumption \ref{asp:basic}(c) formalizes the heavy-tailed noise condition that underpins our analysis, where \eqref{alpha-moment} and \eqref{density-tail} indicate that the noise has a tail index $\alpha\in(0,2]$, and \eqref{truncate-bias} and \eqref{near-sym} serve as additional regularity conditions. We make a few remarks on \eqref{alpha-moment}-\eqref{near-sym} below:

\begin{itemize}
    \item Condition \eqref{alpha-moment} indicates that the $\alpha$th moment of the noise is bounded, which is common in previous studies of stochastic optimization under heavy-tailed noise (e.g., see \cite{zhang2020adaptive}), although main prior work has only considered the case $\alpha\in(1,2]$. In addition, condition \eqref{density-tail} implies that the density function decays according to a power law at the rate of $\mathcal{O}(|u|^{-(\alpha+1)})$, which is useful for modeling distributions with heavy-tailed behavior (e.g., \cite{stumpf2012critical}). As will be shown shortly in Proposition \ref{lem:htbd-density}, conditions \eqref{alpha-moment} and \eqref{density-tail} are nearly equivalent under certain regularity conditions. We impose both conditions solely for convenience of presentation. 
    \item Condition \eqref{truncate-bias} represents asymptotic unbiasedness and generalizes the common unbiasedness condition for finite-mean noise to infinite-mean noise. To see this, one can interchange the limit and the supremum in \eqref{truncate-bias} for the finite-mean case. Condition \eqref{near-sym} implies that the integral over the tail $[\tau,\infty)$ of the symmetry measure $p_{i,x}(\cdot)-p_{i,x}(-\cdot)$ decays at the rate $o(\tau^{-1})$. It holds for symmetric distributions, e.g., the standard Cauchy distribution. For arbitrary heavy-tailed distributions with tail index $\alpha\in(1,2]$, it also holds regardless of whether the distribution is nearly symmetric, as shown in Lemma \ref{lem:bias-asym}.
\end{itemize}    
\end{remark}

The following proposition connects the bounded finite-moment condition with the power-law decay of the noise density. Its proof is deferred to Section \ref{subsec:pf-not}.

\begin{proposition}\label{lem:htbd-density}
Let $\alpha\in(0,2]$ be given, and $\zeta$ be a real-valued random variable with the density function $p(\cdot)$. Then the following statements hold.
\begin{enumerate}
\item[{{(i)}}] Suppose that there exist $z_1>0$ and $M>0$ such that $p(z)\le M/|z|^{\alpha+1}$ for all $|z|\ge z_1$. Then, $\E[|\zeta|^{\beta}]$ is finite for any $\beta\in(0,\alpha)$.
\item[{{(ii)}}] Suppose that $\E[|\zeta|^\alpha]\le M$ for some $M>0$, and $p$ is eventually monotone, i.e., there exists some $z_1>0$ such that $p$ is nonincreasing over $[z_1,\infty)$ and nondecreasing over $(-\infty,-z_1]$. Then, we have $p(z) \le M (2/|z|)^{\alpha+1}$ for all $|z|\ge 2z_1$.
\end{enumerate} 
\end{proposition}

\begin{remark}
Eventual monotonicity in Proposition \ref{lem:htbd-density}(ii) has commonly been used as a regularity condition in the study of heavy-tailed distributions; see Theorem 2.5 in \cite{nair2013fundamentals}.
\end{remark}

The next lemma shows that if a density function has a power-law tail decaying at the rate $\mathcal{O}(|z|^{-(\alpha+1)})$ with $\alpha \in (1,2]$, then the convergence in \eqref{truncate-bias} and \eqref{near-sym} holds with rate $\mathcal{O}(\tau^{-(\alpha-1)})$. Its proof is relegated to Section \ref{subsec:pf-bias-var}.

\begin{lemma}\label{lem:bias-asym}
Let $\alpha\in(1,2]$, and $\zeta$ be a real-valued random variable with mean zero and the density function $p(\cdot)$. Suppose that there exist some $z_1>0$ and $M>0$ such that $p(z)\le M/|z|^{\alpha+1}$ for all $|z|\ge z_1$. Then,
\begin{align}\label{decay-rate-a}
\Big|\int^\tau_{-\tau} zp(z)\mathrm{d}z \Big|\le \frac{2M}{(\alpha-1)\tau^{\alpha-1}},\quad \Big|\tau\int^\infty_\tau(p(z)-p(-z))\mathrm{d}z\Big|\le \frac{2M}{\alpha \tau^{\alpha-1}} \qquad\forall \tau\ge z_1. 
\end{align}
\end{lemma}

We remark that Assumption \ref{asp:basic} is sufficient, as shown in Section \ref{sec:complexity}, to ensure that the clipped SPGMs obtain approximate solutions to problem \eqref{ucpb} for any target tolerance $\epsilon\in(0,1)$ within a finite number of iterations. To derive oracle complexity with explicit dependence on the tolerance $\epsilon$ and the tail index $\alpha$, we make the following additional regularity assumption on the decay rates associated with \eqref{truncate-bias} and \eqref{near-sym}. Under this assumption, we can establish tighter complexity results (see Table \ref{table:complexity}) and, more importantly, handle problems with tail index $\alpha\in(0,1]$ with explicit oracle complexity, a regime that, to the best of our knowledge, has not been studied previously.

\begin{assumption}\label{asp:decay}
There exist $\Gamma_1,\Gamma_2>0$ and $\tau_{(2)}\ge\tau_{(1)}$ such that the density functions $\{p_{i,x}\}$ satisfy 
\begin{equation*}
\sup_{\substack{x \in \mathrm{dom}\,h, \\ 1 \le i \le n}}\Big\{\Big|\int^\tau_{-\tau}up_{i,x}(u)\mathrm{d}u \Big|\Big\}\le \frac{\Gamma_1}{\tau^{\alpha}},\quad \sup_{\substack{x \in \mathrm{dom}\,h, \\ 1 \le i \le n}}\Big\{\Big|\tau\int^\infty_\tau(p_{i,x}(u)-p_{i,x}(-u))\mathrm{d}u\Big|\Big\}\le\frac{\Gamma_2}{\tau^{\alpha}}
\end{equation*}
for all $\tau\ge\tau_{(2)}$, where $\alpha\in(0,2]$ and $\tau_{(1)}>0$ are given in Assumption \ref{asp:basic}(c).
\end{assumption}

\begin{remark}
(i) Assumption \ref{asp:decay} indicates that \eqref{truncate-bias} and \eqref{near-sym} converge with a sublinear rate of convergence. It is readily seen that this assumption applies to noise distributions with symmetric densities, including the standard Cauchy and Lévy stable distributions. Moreover, it applies to noise with a controlled measure of near symmetry and a potentially infinite mean.

(ii) It is not restrictive to assume that the index $\alpha$ in Assumption \ref{asp:decay} coincides with that in Assumption~\ref{asp:basic}(c). Indeed, if Assumptions \ref{asp:basic}(c) and \ref{asp:decay} hold with different indices $\alpha_1,\alpha_2\in(0,2]$, then both assumptions also hold with $\alpha=\min\{\alpha_1,\alpha_2\}$. 
\end{remark}

We next give formal definitions for approximate solutions of problem \eqref{ucpb}.

\begin{definition}\label{def:sol}
Let $\epsilon\in(0,1)$. We say that 
\begin{itemize}
\item $x\in\R^n$ is an $\epsilon$-stochastic optimal solution of \eqref{ucpb} if it satisfies $\E[F(x)-F^*]\le\epsilon$;
\item $x\in\R^n$ is an $\epsilon$-stochastic stationary point of \eqref{ucpb} if it satisfies $\E[\mathrm{dist}^2(0,\partial F(x))]\le\epsilon^2$.
\end{itemize}
\end{definition}



\section{Bias-Variance Trade-off for Clipping}\label{sec:bias-var}

In this section, we establish the bias-variance trade-off for clipping under noise with varying tail indices, including noise with an infinite mean.

To motivate the bias-variance trade-off, we recall that the classical analysis of SFOMs under light-tailed noise typically assumes that the noise has zero mean and finite variance. Nonetheless, when the noise is heavy-tailed, its variance, and sometimes even its mean, is infinite. 
If one applies clipping to heavy-tailed stochastic gradients to trim extreme values, the resulting gradient has finite mean and variance. However, while clipping controls both the mean and variance to be finite, it inevitably introduces bias: the clipped stochastic gradient cannot inherit the unbiasedness of the original stochastic gradient (before clipping). If one views the original stochastic gradient as applying clipping with an infinite threshold, then the variance increases with the clipping threshold, while the bias decreases. This raises the following questions:
\begin{itemize}
\item What is the quantitative relationship between the bias and the clipping threshold, as well as between the variance and the clipping threshold?
\item By adjusting the clipping threshold, one can trade off variance and bias. How does this bias-variance trade-off affect the oracle complexity of algorithms?
\end{itemize}
We address the first question in this section and answer the second by applying the bias-variance trade-off to establish the oracle complexity of clipped SPGMs in the next section.

Before delving into a formal analysis of the bias-variance trade-off, we provide an illustrative visualization of how the bias and variance vary with the clipping threshold for noise with different tail indices. Specifically, we denote the bias and variance of a one-dimensional clipped estimator as follows:
\begin{align}\label{def:bias-variance}
B(\tau) = \big|\E\big[\Pi_{[-\tau,\tau]}(a+\zeta)\big] - a\big|,\quad V(\tau) = \E\big[\big(\Pi_{[-\tau,\tau]}(a+\zeta) - a \big)^2\big],
\end{align}
where $a\in\R$ is a constant, $\zeta$ is a random variable modeling noise, and $\tau>0$ is the clipping threshold. We estimate the expectations in \eqref{def:bias-variance} by taking the average of randomly generated samples, and Figure \ref{fig:bias-var-trade-off} shows the trends of bias and variance as functions of the clipping threshold under standard Gaussian, symmetric Lévy stable, and standard Cauchy noise. We can observe from this figure that, as the clipping threshold increases from 0 to 100, the bias approaches zero while the variance increases toward infinity, except in the Gaussian noise case, where the variance remains finite as the threshold goes to infinity. Observe that the tail indices increase from left to right: $2$, $1.5$, $1$, and $0.5$. Moreover, the variance for the last three heavy-tailed noise cases shows different growth patterns---concave, linear, and convex, respectively---meaning the increase becomes steeper from left to right.
\begin{figure}[ht]
\includegraphics[width=\linewidth]{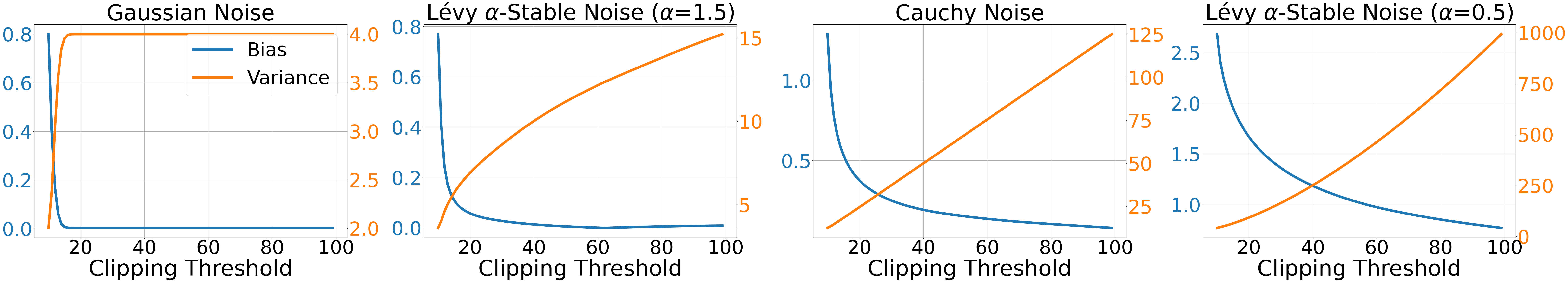}
\label{fig:bias-var-trade-off}
\caption{Bias-Variance Trade-offs under Different Noise}
\end{figure}

Figure \ref{fig:bias-var-trade-off} shows clear trends in bias and variance with respect to the clipping threshold. For a rigorous analysis, the following lemma establishes upper bounds for bias and variance as functions of the clipping threshold. Its proof is relegated to Section \ref{subsec:pf-bias-var}.
\begin{lemma}\label{lem:bias-var-decomp}
Let $\zeta$ be a real-valued random variable with the density function $p(\cdot)$, and $a\in\R$ and $\alpha\in(0,2]$ be given. Suppose that $\E[|\zeta|^\alpha]\le M_1$ holds for some $M_1>0$, and that there exists $z_1>0$ and $M_2>0$ such that $p(z)\le M_2/|z|^{\alpha+1}$ for all $|z|\ge z_1$. Then, for all $\tau\ge z_1 + |a|$, we have
\begin{align}
\big|\E\big[\Pi_{[-\tau,\tau]}(a+\zeta)\big] - a\big| & \le \Big|\int_{-\tau}^\tau zp(z)\mathrm{d}z\Big| + \Big|\tau\int_\tau^\infty (p(z) - p(-z)) \mathrm{d}z\Big| +  \frac{2M_2|a|}{(\tau-|a|)^{\alpha}}\Big(\frac{|a|}{\tau - |a|} + \frac{1}{\alpha}\Big),\label{ineq:1dim-bias}\\
\E\big[\big(\Pi_{[-\tau,\tau]}(a+\zeta)- a\big)^2\big] & \le M_1 (\tau+|a|)^{2-\alpha} + \frac{2M_2(\tau^2 + a^2)}{\alpha(\tau-|a|)^\alpha}.\label{ineq:1dim-var}
\end{align}
\end{lemma}

In view of Lemma \ref{lem:bias-var-decomp}, we observe that for a clipped stochastic estimator, its bias $B(\tau)$ and variance $V(\tau)$, defined in \eqref{def:bias-variance}, satisfy the following upper bounds:
\begin{align*}
B(\tau) \le \Big|\int_{-\tau}^\tau zp(z)\mathrm{d}z\Big| + \Big|\tau\int_\tau^\infty (p(z) - p(-z)) \mathrm{d}z\Big| + \mathcal{O}(\tau^{-\alpha}),\quad V(\tau) \le \mathcal{O}(\tau^{2-\alpha}).  
\end{align*}
We further notice that
\begin{itemize}
    \item The bias can be decomposed into: the truncated expectation of the noise over $[-\tau, \tau]$, a symmetry measure of the tail over $[\tau, \infty]$, and a diminishing term. By this and Lemma \ref{lem:bias-asym}, one can see that $B(\tau)\to0$ as $\tau\to\infty$ when $\alpha\in(1,2]$. Also, when the first two terms in the upper bound for $B(\tau)$ converge to zero, we obtain $B(\tau)\to0$ as $\tau\to\infty$. In either case, the clipped estimator becomes nearly unbiased when the clipping threshold is sufficiently large.
    \item The growth curve of the variance reflects the patterns shown in Figure \ref{fig:bias-var-trade-off}; specifically, it is concave, linear, and convex for $\alpha\in(0,1)$, $\alpha=1$, and $\alpha\in(1,2)$, respectively.
\end{itemize}

We are now ready to analyze the bias-variance trade-off for the clipped stochastic gradient. We first introduce the bias of the clipped stochastic gradient, and the set of clipping thresholds that enforce a sufficiently small bias, as follows:
\begin{align}
\Delta(\tau)& = \sup_{x\in\mathrm{dom}\,h}\big\{\|\E[G_{\tau}(x;\xi)]-\nabla f(x)\|\big\}\qquad\forall \tau\ge0,\label{def:Delta-tau}\\
\mathcal{T}(\varepsilon)& = \{\tau\ge \tau_{(1)}: \Delta(\tau)\le \varepsilon\}\qquad\forall \varepsilon>0,\label{def:T-eps}
\end{align}
where $\tau_{(1)}$ is defined in \eqref{def:tau-underline}. For convenience, we define an upper bound for the variance as
\begin{align}\label{def:sigma-tau}
\sigma^2(\tau) = n\bigg[\Lambda_1 (\tau + U_f)^{2-\alpha} + \frac{2\Lambda_2(\tau^2 + U_f^2)}{\alpha(\tau -U_f)^\alpha}\bigg]\qquad\forall \tau\ge\tau_{(1)},    
\end{align}
where $\alpha,\Lambda_1,\Lambda_2$ are given in Assumption \ref{asp:basic}(c), and $U_f$ is defined in \eqref{def:Uf-Dh}.

The following theorem shows that under Assumption \ref{asp:basic}(c), $\mathcal{T}(\varepsilon)$ is nonempty for any $\varepsilon>0$ and the variance of the clipped stochastic gradient is bounded by $\sigma^2(\tau)$. Moreover, when $\alpha\in(1,2]$, there exists $\tau_1(\varepsilon)\in\mathcal{T}(\varepsilon)$ with $\tau_1(\varepsilon)=\mathcal{O}(\varepsilon^{-1/(\alpha-1)})$. Its proof is relegated to Section \ref{subsec:pf-bias-var}.

\begin{theorem}\label{thm:trade-off1}
Suppose that Assumption \ref{asp:basic}(c) holds. Let $u_1$, $\alpha$, $\Lambda_1$, and $\Lambda_2$ be given in Assumption \ref{asp:basic}(c), and $\tau_{(1)}$, $U_f$, $\mathcal{T}(\cdot)$ and $\sigma^2(\cdot)$ be defined in \eqref{def:Uf-Dh}, \eqref{def:tau-underline}, \eqref{def:T-eps} and \eqref{def:sigma-tau}, respectively. Then the following statements hold.
\begin{enumerate}
\item[{\rm (i)}] $\mathcal{T}(\varepsilon)$ is nonempty for any $\varepsilon>0$, and moreover, 
\begin{align}  \label{ineq:upbd-var-lem}
\sup_{x\in\mathrm{dom}\,h}\big\{\E\big[\big\|G_{\tau}(x;\xi) - \nabla f(x)\big\|^2\big]\big\} \le \sigma^2(\tau) \qquad\forall \tau\ge\tau_{(1)}.  
\end{align}
\item[{\rm (ii)}] When $\alpha\in(1,2]$, for all $\varepsilon\in(0,1)$, we have $\tau_{1}(\varepsilon)\in \mathcal{T}(\varepsilon)$ with $\tau_{1}(\varepsilon)$ defined as
\begin{align}\label{def:tau1-eps}  
\tau_{1}(\varepsilon) = \max\Big\{\tau_{(1)},\Big(\frac{6\sqrt{n}\Lambda_2}{(\alpha-1)\varepsilon}\Big)^{\frac{1}{\alpha-1}},U_f + \Big(\frac{4\sqrt{n}\Lambda_2 U_f\tau_{(1)}}{u_1\varepsilon}\Big)^{\frac{1}{\alpha}}\Big\}.
\end{align}
\end{enumerate}
\end{theorem}

The following theorem shows that if we further impose Assumption \ref{asp:decay}, there exists $\tau_2(\varepsilon)\in\mathcal{T}(\varepsilon)$ with $\tau_2(\varepsilon)=\mathcal{O}(\varepsilon^{-1/\alpha})$, which improves upon the rate in \eqref{def:tau1-eps}. Its proof is deferred to Section \ref{subsec:pf-bias-var}.

\begin{theorem}\label{thm:trade-off2}
Suppose that Assumptions \ref{asp:basic}(c) and \ref{asp:decay} hold. Let $u_1$, $\alpha$, $\Lambda_1$, and $\Lambda_2$ be given in Assumption \ref{asp:basic}(c), $\tau_{(2)}$, $\Gamma_1$, and $\Gamma_2$ be given in Assumption \ref{asp:decay}, and $U_f$, $\mathcal{T}(\cdot)$ and $\sigma^2(\cdot)$ be defined in \eqref{def:Uf-Dh}, \eqref{def:T-eps} and \eqref{def:sigma-tau}, respectively. Then, for all $\varepsilon\in(0,1)$, we have $\tau_2(\varepsilon)\in\mathcal{T}(\varepsilon)$ with $\tau_2(\varepsilon)$ defined as
\begin{align}\label{def:tau-eps-2}
\tau_{2}(\varepsilon) := \max\Big\{\tau_{(2)}, \Big(\frac{2\sqrt{n}(\Gamma_1+\Gamma_2)}{\varepsilon}\Big)^{\frac{1}{\alpha}},U_f + \Big(\frac{4\sqrt{n}\Lambda_2 U_f(U_f/u_1 + 1/\alpha)}{\varepsilon}\Big)^{\frac{1}{\alpha}}\Big\}.
\end{align}    
\end{theorem}

\section{Complexity of Clipped Stochastic Proximal Gradient Methods}\label{sec:complexity}

In this section, we use the bias-variance trade-off to establish the first-order oracle complexity of a clipped SPGM for solving convex problems and a clipped SGPM with momentum for nonconvex problems.

\subsection{Complexity of a Clipped SPGM for Convex Problems}\label{subsec:complexity-cvx}

In this subsection, we establish the first-order oracle complexity of a clipped SPGM for solving convex and strongly convex problems, respectively. Throughout this subsection, we let $x^*$ denote an optimal solution of \eqref{ucpb}, and we make the following assumption regarding the convexity of $f$.
\begin{assumption}\label{asp:cvx}
The function $f$ is convex on $\mathrm{dom}\,h$, i.e., there exists $\mu_f\ge0$ such that 
\begin{align*}
f(y)\ge f(x) + \nabla f(x)^T(y-x) + \frac{\mu_f}{2}\|y-x\|^2\qquad\forall x,y\in\mathrm{dom}\,h.    
\end{align*}    
\end{assumption}

The clipped SPGM is an extension of stochastic approximation algorithms (see, e.g., \cite{nemirovski2009robust}) in which stochastic gradients are replaced by clipped stochastic gradients, allowing the algorithm to handle problem \eqref{ucpb} in the presence of heavy-tailed noise. In particular, this clipped SPGM generates two sequences, $\{x^k\}$ and $\{z^k\}$. At each iteration $k\ge0$, the clipped SPGM first updates $x^{k+1}$ by performing a proximal operation on a clipped stochastic gradient step. It then computes $z^{k+1}$ as an average of the past iterates $x^1,\ldots,x^{k+1}$. The details of this method are presented in Algorithm \ref{alg:c-spgm}, with specific choices of step sizes and clipping thresholds provided in Theorems \ref{thm:cmplex-cvx} and \ref{thm:cmplex-scvx}.


\begin{algorithm}[!htbp]
\caption{A clipped SPGM}
\label{alg:c-spgm}
\begin{algorithmic}[0]
\State \textbf{Input:} starting point $x^0\in\mathrm{dom}\,h$, step sizes $\{\eta_k\}\subset(0,\infty)$, clipping threshold $\tau>0$.
\For{$k=0,1,2,\ldots$}
\State Update the next iterate:
\begin{align}\label{prox-step}
x^{k+1} = \mathrm{prox}_{\eta_k h}(x^k - \eta_k G_\tau(x^k;\xi_k)).
\end{align}
\State Compute the average:
\begin{align}\label{alg1-ave-step}
z^{k+1} = \frac{1}{k+1} \sum_{t=0}^k x^{t+1}.    
\end{align}
\EndFor
\end{algorithmic}
\end{algorithm}

The following lemma gives an upper bound on the expected objective value gap for iterates generated by Algorithm \ref{alg:c-spgm} under Assumptions \ref{asp:basic} and \ref{asp:cvx}. Its proof is deferred to Section \ref{subsec:pf-complexity-cvx}.
\begin{lemma}\label{lem:upbd-obj-gap}
Suppose that Assumptions \ref{asp:basic} and \ref{asp:cvx} hold. Let $L_f$ be given in Assumption~\ref{asp:basic}, $U_f$ and $D_h$ be defined  in \eqref{def:Uf-Dh}, and $\tau_{(1)}$, $\Delta(\cdot)$ and $\sigma^2(\cdot)$ be defined in \eqref{def:tau-underline}, \eqref{def:Delta-tau} and \eqref{def:sigma-tau}, respectively. Let $\{x^k\}$ be the sequence generated by Algorithm~\ref{alg:c-spgm} with the step size $\eta_k\equiv\eta$ for all $k\ge0$ and the clipping threshold $\tau\ge \tau_{(1)}$. Then, it holds that for all $k\ge0$,    
\begin{align}\label{ineq:descent-cvx}
\E_{\xi_k}[F(x^{k+1})-F^*] \le \frac{\|x^k-x^*\|^2 - \E_{\xi_k}[\|x^{k+1} - x^*\|^2]}{2\eta}    + D_h\Delta(\tau) + \bigg(\frac{L_f^2D_h^2}{4} + \sigma^2(\tau)\bigg)\eta.  
\end{align}
\end{lemma}

\begin{remark}
The relation \eqref{ineq:descent-cvx} is similar to the one that can be established for SPGMs under light-tailed noise (see, e.g., \cite[Lemma 3]{lan2012optimal}). The main differences are that, in our case, the bias is no longer zero and the variance is no longer constant; both the bias $\Delta(\tau)$ and the variance $\sigma^2(\tau)$ depend on the clipping threshold $\tau$.    
\end{remark}

The following theorem provides a complexity bound for Algorithm \ref{alg:c-spgm} to compute an $\epsilon$-stochastic optimal solution of \eqref{ucpb}. Its proof is relegated to Section \ref{subsec:pf-complexity-cvx}.

\begin{theorem}[{{\bf convex}}]\label{thm:cmplex-cvx}
Suppose that Assumptions \ref{asp:basic} and \ref{asp:cvx} hold. Let $\epsilon\in(0,1)$ be arbitrarily chosen, and $K$ be a pre-chosen maximum iteration number for running Algorithm \ref{alg:c-spgm}. Let $L_f$ be given in Assumption \ref{asp:basic}, and $D_h$, $\mathcal{T}(\cdot)$ and $\sigma^2(\cdot)$ be defined in \eqref{def:Uf-Dh}, \eqref{def:T-eps} and \eqref{def:sigma-tau}, respectively. Let
\begin{align}\label{def:tau-eta-cvx}
\tau_\epsilon\in\mathcal{T}\Big(\frac{\epsilon}{2D_h}\Big),\quad \eta_{\epsilon}= \frac{D_h}{[2K(L_f^2 D_h^2/4 + \sigma^2(\tau_{\epsilon}))]^{1/2}}.
\end{align}
Let $\{z^k\}$ be generated by Algorithm \ref{alg:c-spgm} with $\eta_k\equiv\eta_{\epsilon}$ for all $k\ge0$ and $\tau=\tau_{\epsilon}$. Then, $\E[F(z^K) - F^*]\le\epsilon$ for all $K$ satisfying
\begin{align}\label{cvx-K-iter-cmplx}
K \ge \max\bigg\{\frac{8D_h^2(L_f^2 D_h^2/4 + \sigma^2(\tau_{\epsilon}))}{\epsilon^2},1\bigg\}.   
\end{align}
\end{theorem}

\begin{remark}
From Theorem \ref{thm:cmplex-cvx}, we see that Algorithm \ref{alg:c-spgm} has a first-order oracle complexity of $\mathcal{O}(\sigma^2(\tau_\epsilon)/\epsilon^2)$, with $\tau_\epsilon\in\mathcal{T}(\frac{\epsilon}{2D_h})$, for obtaining an $\epsilon$-stochastic optimal solution to convex problems. When $\alpha\in(1,2]$, in view of Theorem \ref{thm:trade-off1}(ii), one can select
\begin{align*}
\tau_\epsilon=\tau_1\Big(\frac{\epsilon}{2D_h}\Big)\in\mathcal{T}\Big(\frac{\epsilon}{2D_h}\Big),    
\end{align*}
and it then follows from the definitions of $\sigma^2(\cdot)$ and $\tau_1(\cdot)$ in \eqref{def:sigma-tau} and \eqref{def:tau1-eps}, respectively, that the complexity becomes $\mathcal{O}(\epsilon^{-\alpha/(\alpha-1)})$, which recovers the best-known results under condition \eqref{asp:heavy-tail}; e.g., see \cite{fatkhullin2025can,he2025accelerated,sadiev2023high}. If Assumption \ref{asp:decay} is further imposed, then one can select
\begin{align*}
\tau_\epsilon=\tau_2\Big(\frac{\epsilon}{2D_h}\Big)\in\mathcal{T}\Big(\frac{\epsilon}{2D_h}\Big),    
\end{align*}
and it follows from the definitions of $\sigma^2(\cdot)$ and $\tau_2(\cdot)$ in \eqref{def:sigma-tau} and \eqref{def:tau-eps-2}, respectively, that the complexity becomes $\mathcal{O}(\epsilon^{-(\alpha+2)/\alpha})$.
\end{remark}

The next lemma provides an upper bound on the expected objective value gap for the iterates generated by Algorithm \ref{alg:c-spgm} when solving strongly convex problems. Its proof is deferred to Section \ref{subsec:pf-complexity-cvx}.

\begin{lemma}\label{lem:upbd-obj-gap-scvx}
Suppose that Assumptions \ref{asp:basic} and \ref{asp:cvx} hold with $\mu_f>0$. Let $L_f$ be given in Assumption~\ref{asp:basic}, and let $D_h$, $\tau_{(1)}$, $\Delta(\cdot)$ and $\sigma^2(\cdot)$ be defined in \eqref{def:Uf-Dh}, \eqref{def:tau-underline}, \eqref{def:Delta-tau} and \eqref{def:sigma-tau}, respectively. Let $\{x^k\}$ be the sequence generated by Algorithm~\ref{alg:c-spgm} with the step sizes $\{\eta_k\}$ and the clipping threshold $\tau\ge\tau_{(1)}$. Then, it holds that for all $k\ge0$,    
\begin{align}
\E_{\xi_k}[F(x^{k+1})-F^*] & \le \Big(\frac{1}{2\eta_k}-\frac{\mu_f}{4}\Big)\|x^k-x^*\|^2 - \frac{1}{2\eta_k}\E_{\xi_k}[\|x^{k+1} - x^*\|^2] + \frac{\Delta^2(\tau)}{\mu_f}\nonumber\\
&\qquad + \bigg(\|\nabla f(x^k) - G_{\tau}(x^k;\xi_k)\|^2 + \frac{L_f^2D_h^2}{4}\bigg)\eta_k. \label{ineq:descent-scvx} 
\end{align}    
\end{lemma}

The following theorem gives a complexity bound for Algorithm \ref{alg:c-spgm-m} to compute an $\epsilon$-stochastic stationary point of \eqref{ucpb}, whose proof is deferred to Section \ref{subsec:pf-complexity-cvx}.

\begin{theorem}[{{\bf strongly convex}}]\label{thm:cmplex-scvx}
Suppose that Assumptions \ref{asp:basic} and \ref{asp:cvx} hold with $\mu_f>0$. Let $\epsilon\in(0,1)$ be arbitrarily chosen, and let $K$ be a pre-chosen maximum iteration number for running Algorithm \ref{alg:c-spgm}. Let $L_f$ be given in Assumption \ref{asp:basic}, and $D_h$, $\mathcal{T}(\cdot)$ and $\sigma^2(\cdot)$ be defined in \eqref{def:Uf-Dh}, \eqref{def:T-eps} and \eqref{def:sigma-tau}, respectively. Let    
\begin{align}\label{def:tilde-tau-eta-scvx}
\tilde{\tau}_\epsilon\in\mathcal{T}\bigg(\sqrt{\frac{\mu_f\epsilon}{2}}\bigg),\quad \tilde{\eta}_k = \frac{2}{\mu_f(k+1)}\qquad\forall k\ge0
\end{align}
Let $\{z^k\}$ be generated by Algorithm \ref{alg:c-spgm} with $\eta_k=\tilde\eta_k$ for all $k\ge0$ and $\tau=\tilde{\tau}_\epsilon$. Then, $\E[F(z^K) - F^*]\le\epsilon$ for all $K$ satisfying
\begin{align}\label{scvx-K-iter-cmplx}
K \ge \max\bigg\{\bigg(\frac{4(L_f^2D_h^2 + 4\sigma^2(\tilde{\tau}_\epsilon))}{\mu_f\epsilon}\bigg) \ln\bigg(\frac{4(L_f^2D_h^2 + 4\sigma^2(\tilde{\tau}_\epsilon))}{\mu_f\epsilon}\bigg),3\bigg\}.   
\end{align}
\end{theorem}

\begin{remark}
From Theorem \ref{thm:cmplex-scvx}, we observe that Algorithm \ref{alg:c-spgm} can achieve a first-order oracle complexity of $\widetilde{\mathcal{O}}(\sigma^2(\tilde\tau_\epsilon)/\epsilon)$ with $\tilde\tau_\epsilon\in\mathcal{T}(\sqrt{\frac{\mu_f\epsilon}{2}})$. When $\alpha\in(1,2]$, in view of Theorem \ref{thm:trade-off1}(ii), one can select
\begin{align*}
\tilde\tau_\epsilon=\tau_1\bigg(\sqrt{\frac{\mu_f\epsilon}{2}}\bigg)\in\mathcal{T}\bigg(\sqrt{\frac{\mu_f\epsilon}{2}}\bigg),    
\end{align*}
and it then follows from the definitions of $\sigma^2(\cdot)$ and $\tau_1(\cdot)$ in \eqref{def:sigma-tau} and \eqref{def:tau1-eps}, respectively, that the complexity becomes $\widetilde{\mathcal{O}}(\epsilon^{-\alpha/(2(\alpha-1))})$, which recovers the best-known results under condition \eqref{asp:heavy-tail}; e.g., see \cite{fatkhullin2025can,he2025accelerated,sadiev2023high,zhang2020adaptive}. If Assumption \ref{asp:decay} is further imposed, then one can select
\begin{align*}
\tilde\tau_\epsilon=\tau_2\bigg(\sqrt{\frac{\mu_f\epsilon}{2}}\bigg)\in\mathcal{T}\bigg(\sqrt{\frac{\mu_f\epsilon}{2}}\bigg), 
\end{align*}
and it follows from the definitions of $\sigma^2(\cdot)$ and $\tau_2(\cdot)$ in \eqref{def:sigma-tau} and \eqref{def:tau-eps-2}, respectively, that the complexity becomes $\widetilde{\mathcal{O}}(\epsilon^{-(\alpha+2)/(2\alpha)})$.
\end{remark}

\subsection{Complexity of a Clipped SPGM with Momentum for Nonconvex Problems}\label{subsec:complexity-ncvx}

In this subsection, we establish the first-order oracle complexity of a clipped SPGM with momentum for solving nonconvex problems.

The clipped SPGM with momentum is an extension of the SPGM with momentum (see, e.g., \cite{gao2024non}) with stochastic gradients replaced by clipped stochastic gradients, which allows the algorithm to handle problem \eqref{ucpb} in the presence of heavy-tailed noise. In particular, the algorithm is first initialized with a search direction $m^0=G_{\tau_0}(x^0;\xi_0)$, and then it generates two sequences, $\{x^k\}$ and $\{z^k\}$. At each iteration $k\ge0$, the clipped SPGM first updates $x^{k+1}$ by performing a proximal operation on a clipped stochastic gradient step. Then the next search direction $m^{k+1}$ is computed as a weighted average of the clipped stochastic gradients of $f$ evaluated at the iterates $x^0,\ldots,x^{k+1}$. The details of this method are presented in Algorithm \ref{alg:c-spgm-m}, with the specific choices of step sizes, weighting parameters, and clipping thresholds provided in Theorem~\ref{thm:cmplex-ncvx}.


\begin{algorithm}[!htbp]
\caption{A clipped SPGM with momentum}
\label{alg:c-spgm-m}
\begin{algorithmic}[0]
\State \textbf{Input:} starting point $x^0\in\mathrm{dom}\,h$, step size $\eta>0$, weighting parameter $\theta\in(0,1]$, clipping thresholds $\{\tau_k\}\subset(0,\infty)$.
\State Set $m^0 = G_{{\tau}_0}(x^0;\xi_{0})$.
\For{$k=0,1,2,\ldots$}
\State Update the next iterate:
\begin{align}\label{prox-step-ncvx}
x^{k+1} = \mathrm{prox}_{\eta h}(x^k - \eta m^k).
\end{align}
\State Compute the search direction:
\begin{align}\label{alg2:momentum-step}
m^{k+1} = (1 - \theta) m^k + \theta G_{\tau_{k+1}}(x^{k+1};\xi_{k+1}).    
\end{align}
\EndFor
\end{algorithmic}
\end{algorithm}

For convenience, we define a sequence of potentials for Algorithm \ref{alg:c-spgm-m} as
\begin{align}\label{def:pot-seq}
    \mathcal{P}_k = f(x^k) + \frac{1}{L_f} \|m^k - \nabla f(x^k)\|^2\qquad\forall k\ge0,
\end{align}
where the sequence $\{(x^k,m^k)\}$ is generated by Algorithm \ref{alg:c-spgm-m}, and $L_f$ is given in Assumption \ref{asp:basic}. The next lemma establishes a descent property for the potential sequence $\{\mathcal{P}_k\}$ defined in \eqref{def:pot-seq}. Its proof is deferred to Section \ref{subsec:pf-complexity-ncvx}.

\begin{lemma}\label{lem:pot-decrease}
Suppose that Assumption \ref{asp:basic} holds. Let $L_f$ be given in Assumption~\ref{asp:basic}, and $\tau_{(1)}$, $\Delta(\cdot)$ and $\sigma^2(\cdot)$ be defined in \eqref{def:tau-underline}, \eqref{def:Delta-tau} and \eqref{def:sigma-tau}, respectively. Let $\{x^k\}$ be the sequence generated by Algorithm~\ref{alg:c-spgm-m} with the step size $\eta\in(0,\frac{1}{4L_f}]$, the weighting parameter $\theta=4L_f\eta$, and the clipping thresholds $\{\tau_k\}\subset[\tau_{(1)},\infty)$. Then, it holds that for all $k\ge0$,
\begin{align}\label{ineq:desc-pot-ncvx}
\E_{\xi_{k+1}}[\mathcal{P}_{k+1}] \le \mathcal{P}_k - \frac{\eta}{16}\mathrm{dist}^2(0,\partial F(x^{k+1})) + 8\eta\Delta^2(\tau_{k+1}) + 16L_f\eta^2 \sigma^2(\tau_{k+1}).
\end{align}
\end{lemma}

\begin{remark}
The descent relation \eqref{ineq:desc-pot-ncvx} resembles the one that can be established for SPGMs under light-tailed noise (see, e.g., \cite[Lemma 5]{gao2024non}), except that the bias is no longer zero and the variance is no longer constant; both the bias $\Delta(\tau)$ and the variance $\sigma^2(\tau)$ depend on the clipping threshold $\tau$.    
\end{remark}

The following theorem provides a complexity bound for Algorithm 3 to compute an $\epsilon$-stochastic stationary point of
\eqref{ucpb}. Its proof is deferred to Section \ref{subsec:pf-complexity-ncvx}.

\begin{theorem}[{{\bf nonconvex}}]\label{thm:cmplex-ncvx}
Suppose that Assumption \ref{asp:basic} holds. Let $\epsilon\in(0,1)$ be arbitrarily chosen, and let $K$ be a pre-chosen maximum iteration number for running Algorithm \ref{alg:c-spgm-m}. Let $L_f$ be given in Assumption \ref{asp:basic}, and $F_{\mathrm{low}}$, $\tau_{(1)}$, $\mathcal{T}(\cdot)$ and $\sigma^2(\cdot)$ be defined in \eqref{def:Uf-Dh}, \eqref{def:tau-underline}, \eqref{def:T-eps} and \eqref{def:sigma-tau}, respectively. Let   
\begin{align}\label{def:hat-eta-theta-tau}
\hat{\tau}_\epsilon \in\mathcal{T}\Big(\frac{\epsilon}{32}\Big),\quad \hat\eta_\epsilon = \frac{1}{4}\min\bigg\{\frac{1}{L_f},\bigg(\frac{F(x^0) -F_{\mathrm{low}} + \sigma^2(\tau_{(1)})/L_f}{KL_f\sigma^2(\hat{\tau}_\epsilon)}\bigg)^{1/2}\bigg\},\quad \hat{\theta}_\epsilon=4L_f\hat{\eta}_\epsilon.
\end{align}
Let $\{x^k\}$ be generated by Algorithm \ref{alg:c-spgm-m} with $(\eta,\theta)=(\hat\eta_\epsilon,\hat\theta_\epsilon)$, $\tau_0=\tau_{(1)}$, and $\tau_k=\hat{\tau}_\epsilon$ for all $k\ge1$. Then, $\E[\mathrm{dist}^2(0,\partial F(x^{\iota_K}))]\le\epsilon^2$ for all $K$ satisfying
\begin{align}\label{K-lwbd-ncvx}
K\ge \max\bigg\{1, \frac{512L_f(F(x^0) -F_{\mathrm{low}} + \sigma^2({\tau}_{(1)})/L_f)}{3\epsilon^2},\frac{256^2L_f(F(x^0) -F_{\mathrm{low}} + \sigma^2({\tau}_{(1)})/L_f)\sigma^2(\hat{\tau}_\epsilon)}{\epsilon^4}\bigg\},    
\end{align}
where $\iota_K$ is uniformly drawn from $\{1,\ldots,K\}$.
\end{theorem}

\begin{remark}
From Theorem \ref{thm:cmplex-ncvx}, we observe that Algorithm \ref{alg:c-spgm-m} can achieve a first-order oracle complexity of $\widetilde{\mathcal{O}}(\sigma^2(\hat\tau_\epsilon)/\epsilon^4)$ with $\hat\tau_\epsilon\in\mathcal{T}(\frac{\epsilon}{32})$. When $\alpha\in(1,2]$, in view of Theorem \ref{thm:trade-off1}(ii), one can select
\begin{align*}
\hat\tau_\epsilon=\tau_1\Big(\frac{\epsilon}{32}\Big)\in\mathcal{T}\Big(\frac{\epsilon}{32}\Big),    
\end{align*}
and it then follows from the definitions of $\sigma^2(\cdot)$ and $\tau_1(\cdot)$ in \eqref{def:sigma-tau} and \eqref{def:tau1-eps}, respectively, that the complexity becomes $\mathcal{O}(\epsilon^{-(3\alpha-2)/(\alpha-1)})$, which recovers the best-known results under condition \eqref{asp:heavy-tail}; e.g., see \cite{fatkhullin2025can,he2025complexity,sadiev2023high,zhang2020adaptive}. If Assumption \ref{asp:decay} is further imposed, then one can select
\begin{align*}
\hat\tau_\epsilon=\tau_2\Big(\frac{\epsilon}{32}\Big)\in\mathcal{T}\Big(\frac{\epsilon}{32}\Big), 
\end{align*}
and it follows from the definitions of $\sigma^2(\cdot)$ and $\tau_2(\cdot)$ in \eqref{def:sigma-tau} and \eqref{def:tau-eps-2}, respectively, that the complexity becomes $\mathcal{O}(\epsilon^{-(3\alpha+2)/\alpha})$.
\end{remark}

\section{Numerical Results}\label{sec:nr}
In this section, we present numerical experiments evaluating the convergence behavior of Algorithms \ref{alg:c-spgm} and \ref{alg:c-spgm-m} under different clipping thresholds and noise levels. Both algorithms are implemented in \textsc{Matlab}, and all computations are performed on a laptop equipped with an Intel Core i9-14900HX processor (2.20 GHz) and 32 GB of RAM.

In implementing Algorithms \ref{alg:c-spgm} and \ref{alg:c-spgm-m}, we simulate noisy gradient evaluations by injecting heavy-tailed noise into the gradients as $G(x;\xi)=\nabla f(x) + \xi$, where $\xi$ is an $n$-dimensional random vector with each coordinate follows a heavy-tailed distribution with tail index $\alpha\in(0,2]$. We generate each coordinate of $\xi$ as $\xi_{(i)}=YU^{-1/\alpha}$, where $Y$ follows a Rademacher distribution (i.e., $Y=1$ with probability $0.5$ and $Y=-1$ with probability $0.5$), and $U$ follows a uniform distribution over $(0,1]$. Then one can verify that the density function of $\xi_{(i)}$ is symmetric with respect to $0$, and also that
\begin{align*}
\mathbb{P}(|\xi_{(i)}| \ge \tau) = \mathbb{P}(U^{-1/\alpha} \ge \tau) = \mathbb{P}(U \le \tau^{-\alpha}) = \tau^{-\alpha}\qquad\forall \tau\ge1.  
\end{align*}
Hence, the density function of $\xi_{(i)}$, denoted by $p_{i}(z)$ is of order $\mathcal{O}(|z|^{-(\alpha+1)})$, and the tail index of $\xi_{(i)}$ is $\alpha$.

\subsection{$\ell_1$-Regularized Convex Regression}
In this subsection, we consider the $\ell_1$-regularized least-squares problem:
\begin{align}\label{pb:lasso}
\min_{l\le x\le u}\ \frac{1}{2}\|Ax-b\|^2 + \lambda\|x\|_1,    
\end{align}
where $A\in\R^{m\times n}$, $b\in\R^n$ with $m=200$ and $n=100$, $u=-l=100\cdot\mathbf{1}$ with $\mathbf{1}$ being the all-ones vector, and $\lambda=1$. We randomly generate $A$ and $b$, with each entry sampled from the standard normal distribution. We apply Algorithm \ref{alg:c-spgm} with different clipping thresholds $\tau>0$ to solve \eqref{pb:lasso} under noise with different tail indices $\alpha\in(0,2]$. For every run of Algorithm \ref{alg:c-spgm} with specific tail index $\alpha$ and clipping threshold $\tau$, we initialize the algorithm at the all-zero vector and tune the step size to optimize its individual performance.

\begin{figure}[ht]
\centering
\includegraphics[width=.24\linewidth]{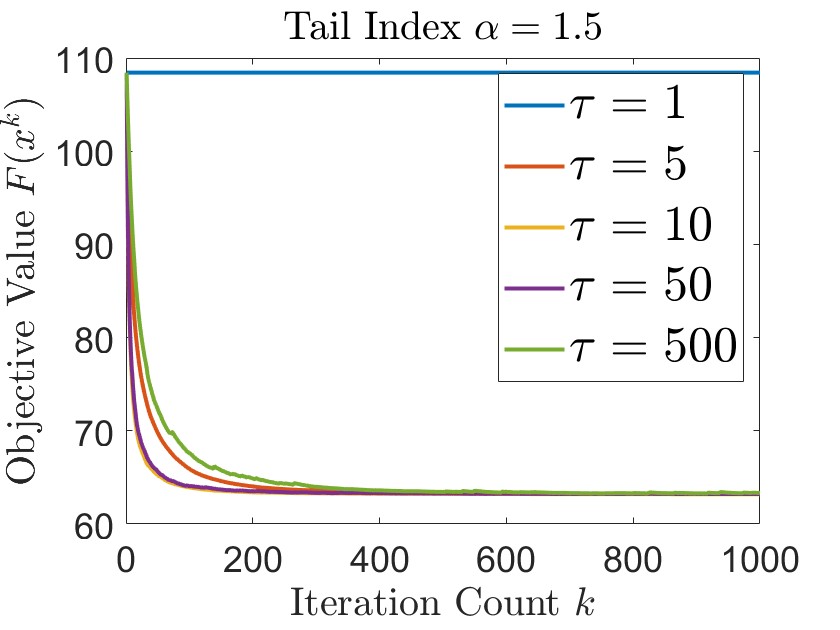}\includegraphics[width=.24\linewidth]{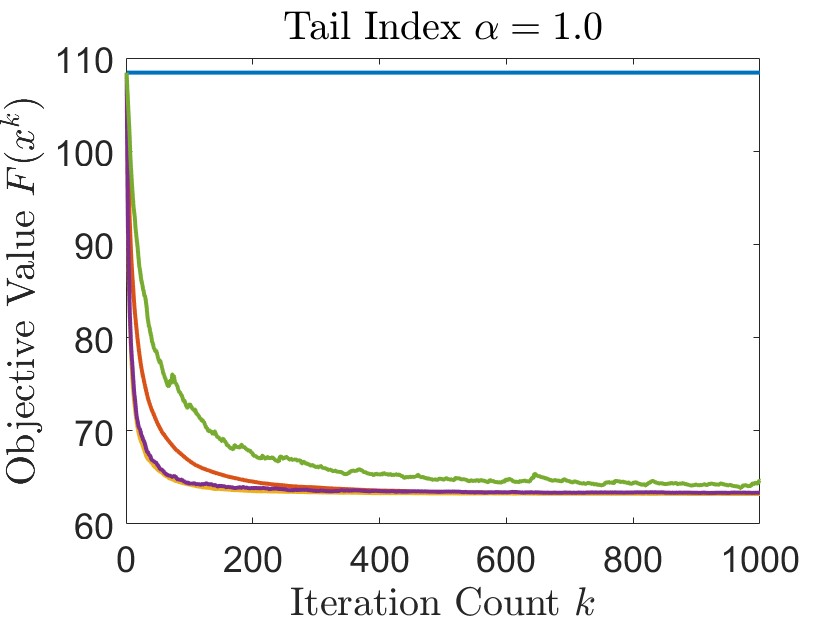}\includegraphics[width=.24\linewidth]{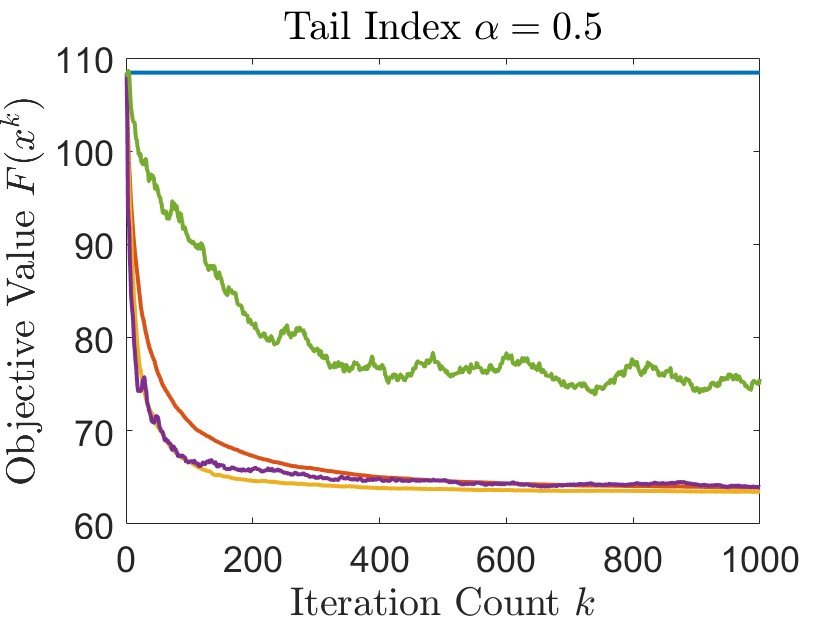}\includegraphics[width=.24\linewidth]{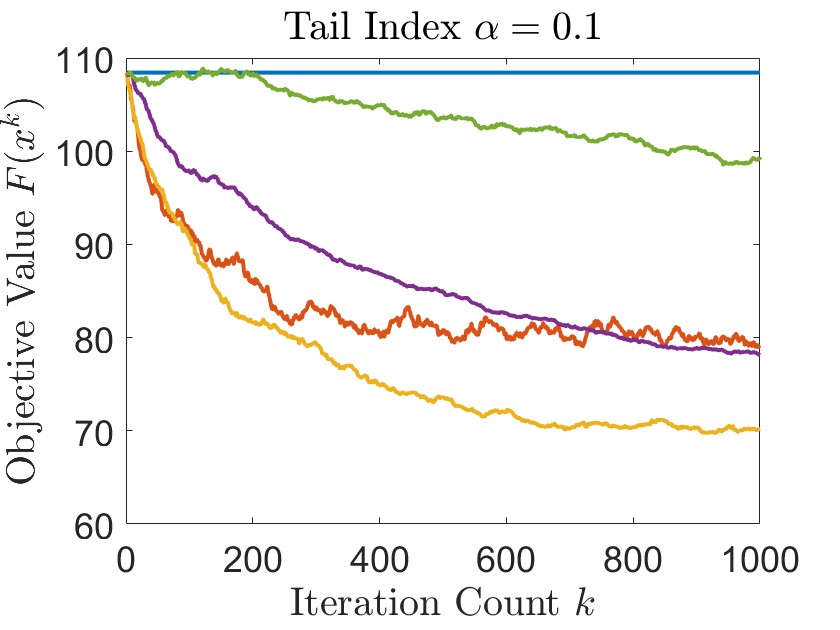}\\
\includegraphics[width=.24\linewidth]{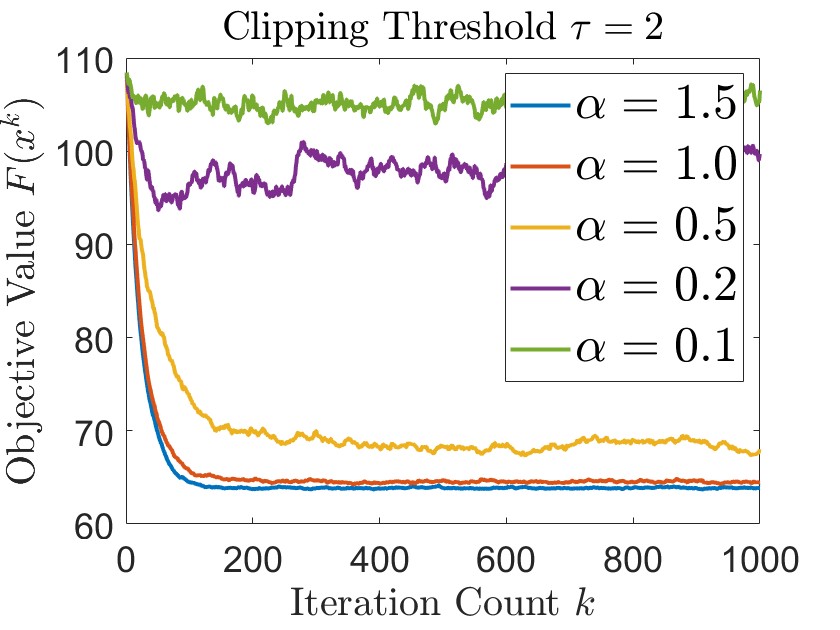}\includegraphics[width=.24\linewidth]{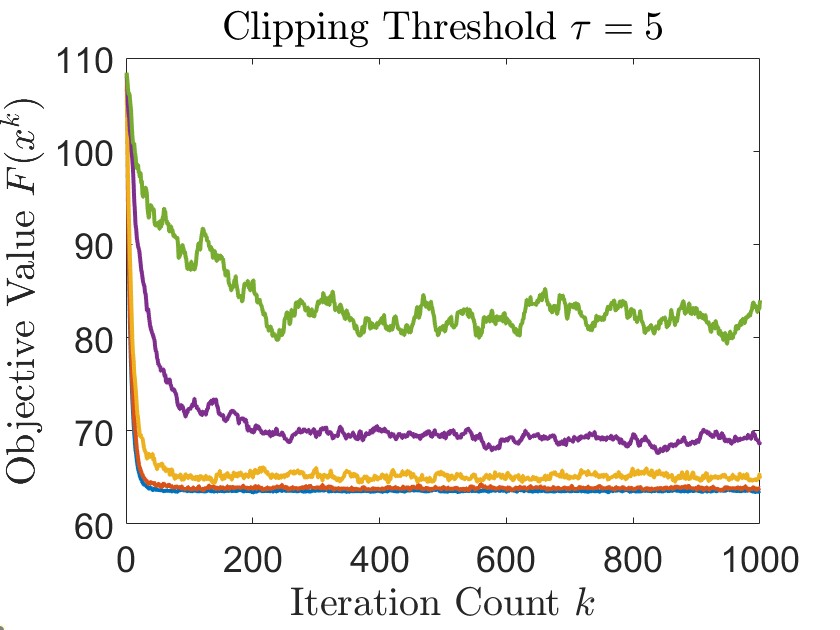}\includegraphics[width=.24\linewidth]{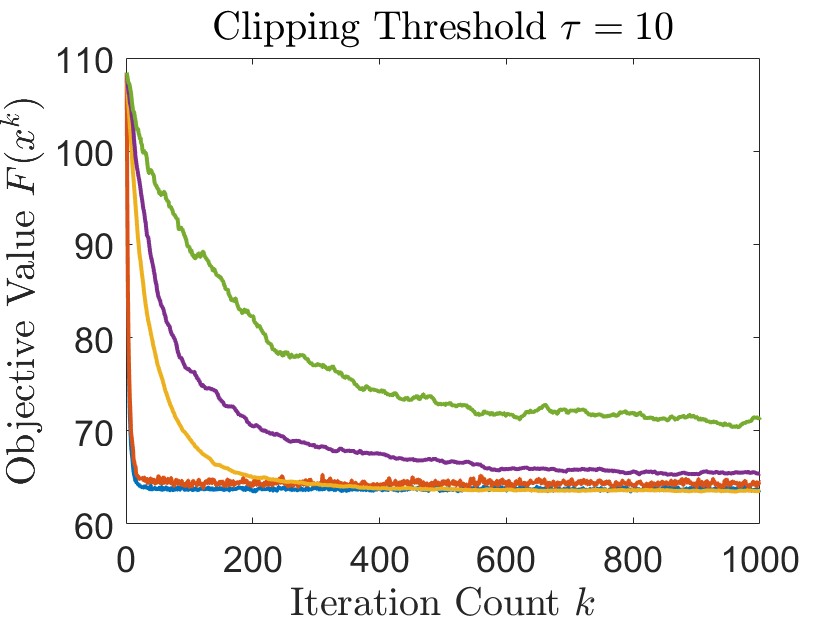}\includegraphics[width=.24\linewidth]{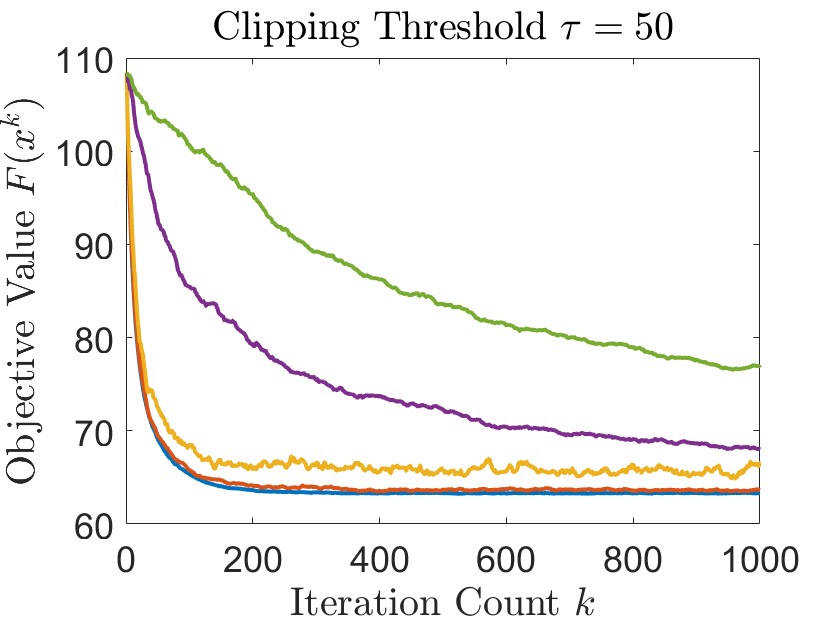}
\caption{Convergence Behavior of Algorithm \ref{alg:c-spgm} for Solving \eqref{pb:lasso} under Different Clipping Thresholds and Noise Levels}
\label{fig:cvx-behave}
\end{figure}

The convergence behavior of Algorithm \ref{alg:c-spgm} is presented in Figure \ref{fig:cvx-behave}. Specifically, the first row presents the convergence behavior with different clipping thresholds $\tau>0$, for fixed heavy-tail indices $\alpha$, and the second row presents the convergence behavior under different noise levels $\alpha\in(0,2]$, for fixed clipping thresholds $\tau>0$. From the first row, we can see that the clipped SPGM can converge if a suitable clipping threshold is applied even in cases where the noise has no finite mean (i.e., $\alpha \le 1$). However, if the clipping threshold is too small, the clipped SPGM fails to reduce the objective value due to the large bias introduced by clipping. Conversely, if the clipping threshold is too large and the noise has no finite mean (i.e., $\alpha \le 1$), the clipped SPGM will also fail to converge, because in this case it is close to the vanilla SPGM, whose convergence appears to be disrupted by the infinite-mean noise. From the second row, we observe that for noise with a less heavy tail (i.e., $\alpha\ge1$), the clipped SPGM converges with all tested clipping thresholds. However, for noise with a heavier tail (i.e., $\alpha\le0.5$), the algorithm performs well only within a narrower range of clipping thresholds. This suggests that, for heavier-tailed noise, the range of suitable clipping thresholds becomes increasingly limited.

\subsection{$\ell_1$-Regularized Nonconvex Regression}
In this subsection, we consider the $\ell_1$-regularized nonconvex regression problem:
\begin{align}\label{pb:ncvx-l1}
\min_{l\le x\le u}\ \sum_{i=1}^m\phi(a_i^Tx-b_i) + \lambda\|x\|_1,    
\end{align}
where $\phi(t)=t^2/(1+t^2)$, $a_i\in\R^{n}$, $b_i\in\R$, $i=1,\ldots,m$, with $m=200$ and $n=100$, $u=-l=100\cdot\mathbf{1}$ with $\mathbf{1}$ being the all-ones vector, and $\lambda=1$. We randomly generate $a_i$'s and $b_i$'s, with each element sampled from the standard normal distribution. We apply Algorithm \ref{alg:c-spgm-m} with different clipping thresholds $\tau>0$ to solve \eqref{pb:ncvx-l1} under noise with different tail indices $\alpha\in(0,2]$. For every run of Algorithm \ref{alg:c-spgm-m} with specific tail index $\alpha$ and clipping threshold $\tau$, we initialize the algorithm at the all-zero vector and tune the step size and weighting parameter to optimize its individual performance.

\begin{figure}[ht]
\centering
\includegraphics[width=.24\linewidth]{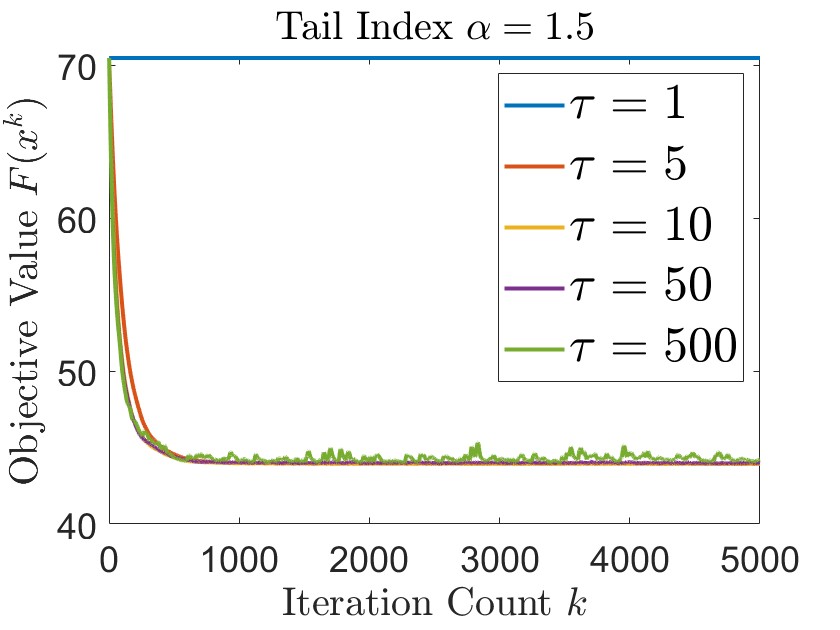}\includegraphics[width=.24\linewidth]{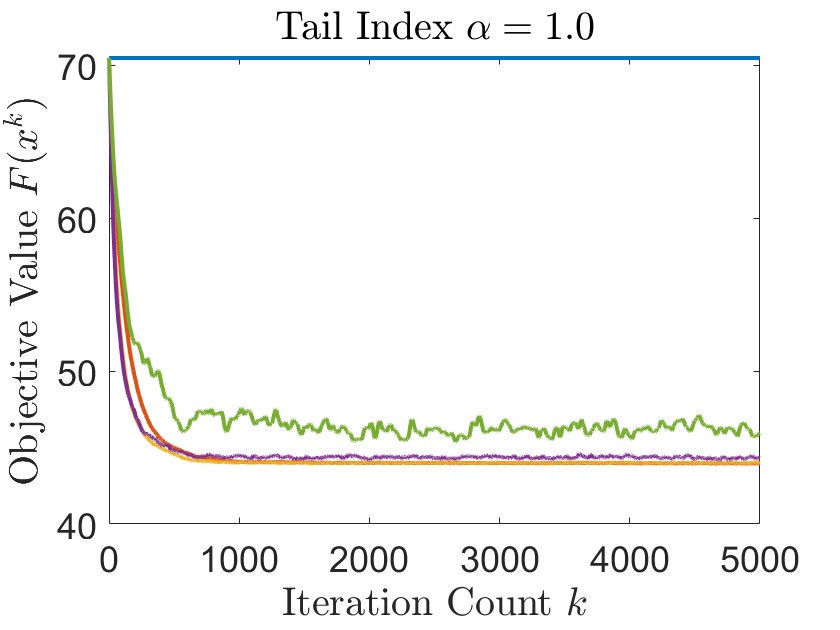}\includegraphics[width=.24\linewidth]{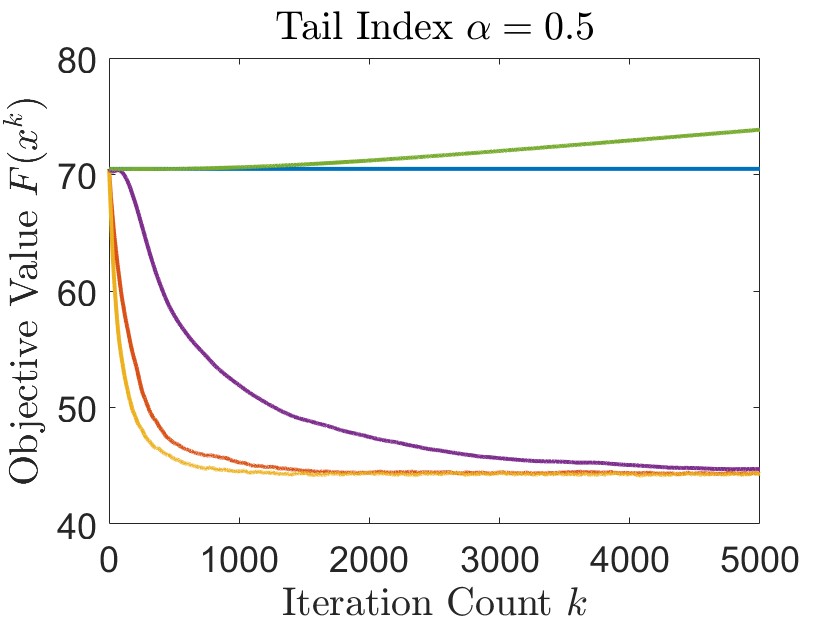}\includegraphics[width=.24\linewidth]{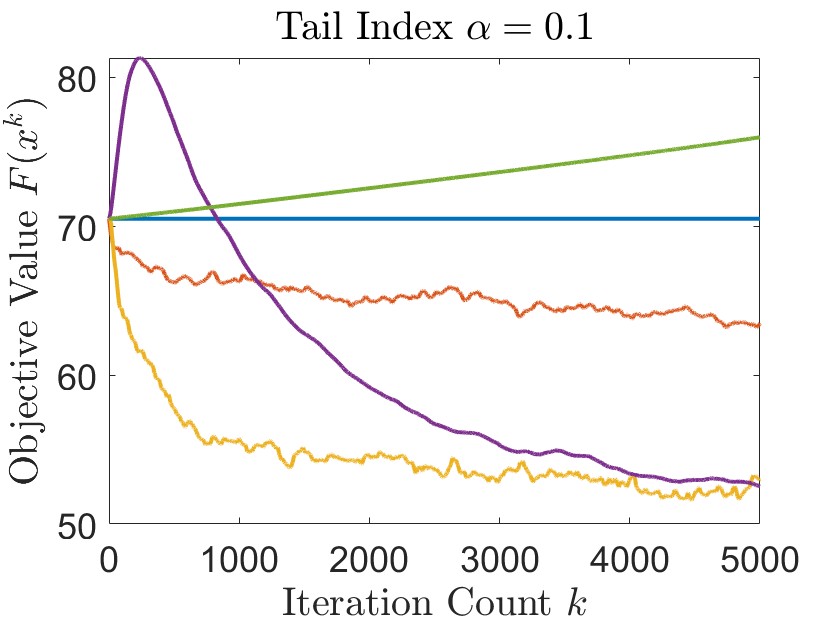}\\
\includegraphics[width=.24\linewidth]{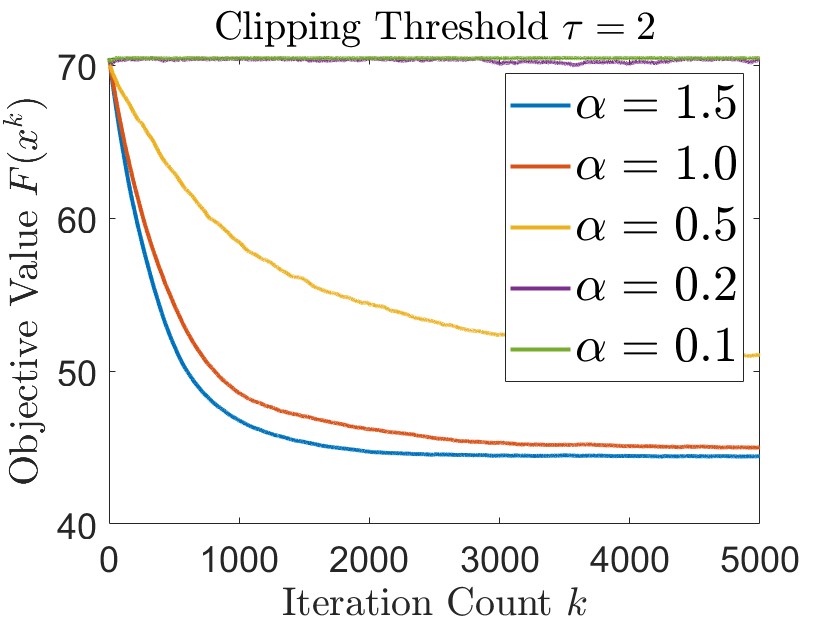}\includegraphics[width=.24\linewidth]{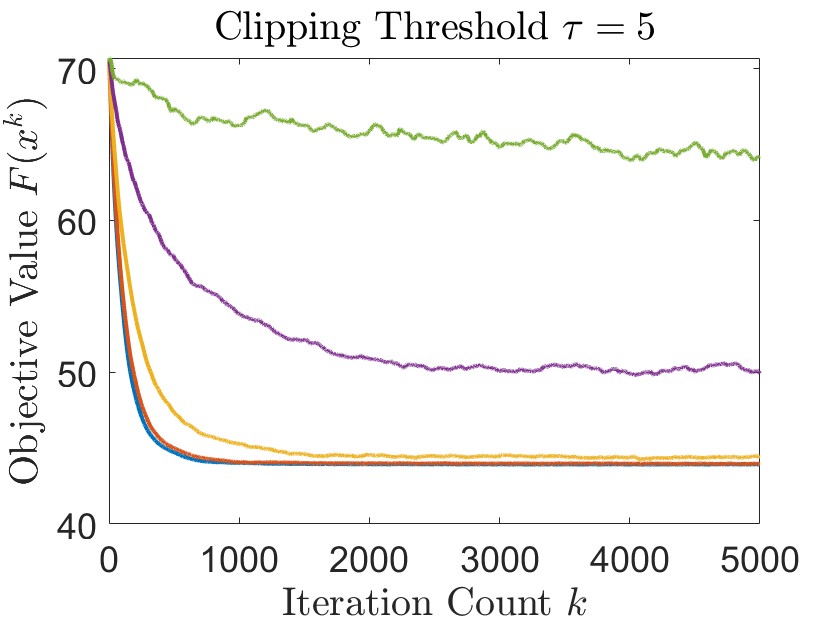}\includegraphics[width=.24\linewidth]{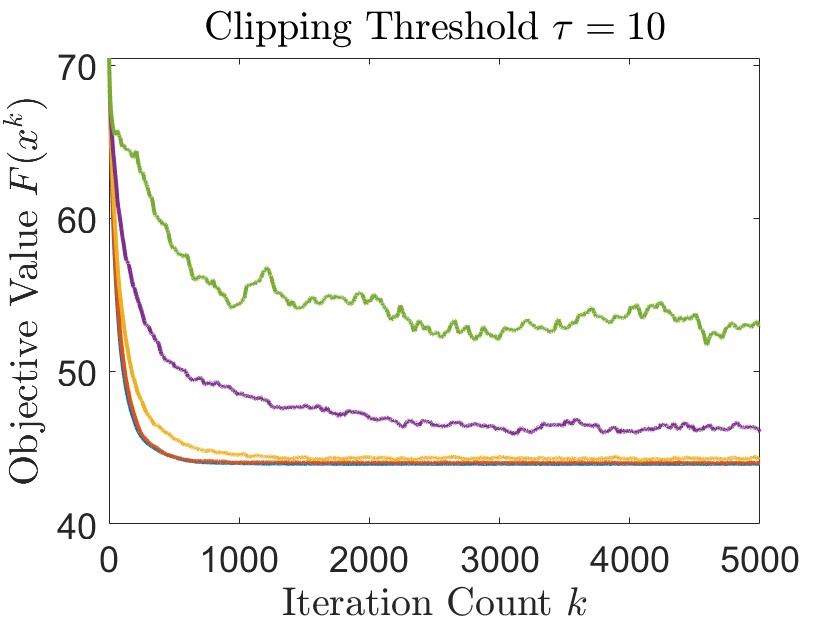}\includegraphics[width=.24\linewidth]{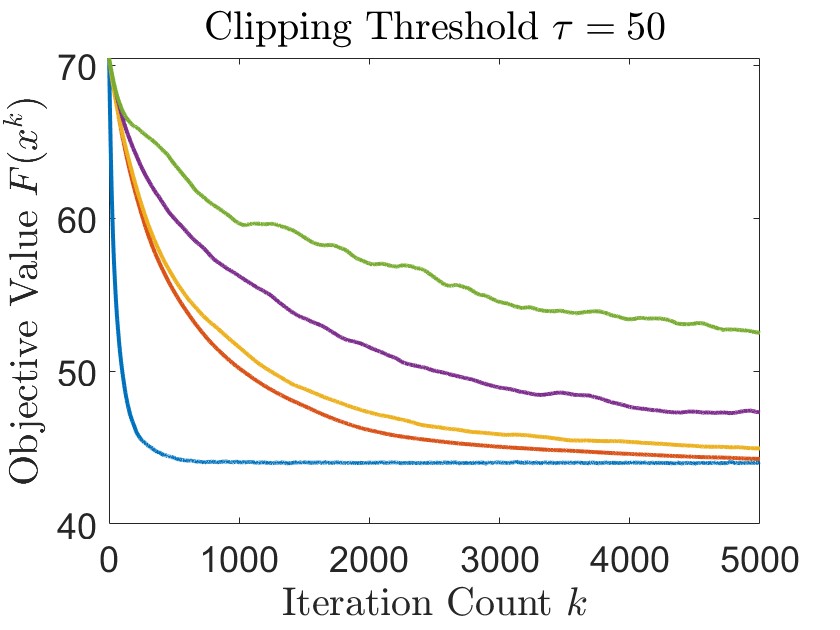}
\label{fig:ncvx-behave}
\caption{Convergence Behavior of Algorithm \ref{alg:c-spgm-m} for Solving \eqref{pb:ncvx-l1} under Different Clipping Thresholds and Noise Levels}
\end{figure}

The convergence behavior of Algorithm \ref{alg:c-spgm-m} is presented in Figure \ref{fig:ncvx-behave}. Specifically, the first row presents the convergence behavior with different clipping thresholds $\tau>0$, for fixed heavy-tail indices $\alpha$, and the second row presents the convergence behavior under different noise levels $\alpha\in(0,2]$, for fixed clipping thresholds $\tau>0$. From the first row, we observe that the clipped SPGM with momentum can converge when a moderate clipping threshold is applied, even in cases where the noise has no finite mean (i.e., $\alpha\le1$). However, if the clipping threshold is too small, the algorithm fails to reduce the objective value due to the large bias introduced by clipping. On the other hand, if the clipping threshold is too large, the clipped SPGM with momentum becomes close to the vanilla SPGM, which fails to converge or may even diverge under heavier-tailed noise with an infinite mean (i.e., $\alpha\le0.5$). From the second row, we see that when the noise has a less heavy tail (i.e., $\alpha\ge1$), the clipped SPGM with momentum converges across the tested clipping thresholds selected from $[2,50]$. In contrast, for noise with a heavier tail (i.e., $\alpha\le0.5$), the algorithm performs well only within a limited range of thresholds. This is partly because a too-small clipping threshold introduces large bias, while a too-large threshold induces excessive variance; therefore, a moderate clipping threshold is required to ensure optimal performance. These observations also indicate that as the noise becomes heavier-tailed, the interval of suitable clipping thresholds may become increasingly narrow.

\section{Proof of the Main Results}\label{sec:pf}

In this section, we provide the proofs of our main results presented in Sections \ref{sec:not}, \ref{sec:bias-var}, and \ref{sec:complexity}, specifically, Proposition \ref{lem:htbd-density}, Lemmas \ref{lem:bias-asym} to \ref{lem:pot-decrease}, and Theorems \ref{thm:trade-off1} to \ref{thm:cmplex-ncvx}.

\subsection{Proof of the Main Results in Section \ref{sec:not}}\label{subsec:pf-not}

In this subsection, we prove Proposition \ref{lem:htbd-density} and Lemma \ref{lem:bias-asym}.

\begin{proof}[\textbf{Proof of Proposition \ref{lem:htbd-density}}]
We first prove statement (i). Fix any $\beta\in(0,\alpha)$. Then one has that
\begin{align*}
\E[|\zeta|^\beta] & = \int^\infty_{-\infty} |z|^\beta p(z) \mathrm{d}z = \int^{z_1}_{-z_1} |z|^\beta p(z) \mathrm{d}z + \int_{|z|> z_1} |z|^\beta p(z)\mathrm{d}z  \\
&\le \int^{z_1}_{-z_1} |z|^\beta p(z) \mathrm{d}z + \int_{|z|\ge z_1} \frac{M}{|z|^{\alpha-\beta+1}}\mathrm{d}z = \int^{z_1}_{-z_1} |z|^\beta p(z) \mathrm{d}z + 2\int_{z\ge z_1} \frac{M}{z^{\alpha-\beta+1}}\mathrm{d}z\\
&= \int^{z_1}_{-z_1} |z|^\beta p(z) \mathrm{d}z + \frac{2M}{(\alpha-\beta)z_1^{\alpha-\beta}} < \infty,
\end{align*}
where the inequality follows from $p(z)\le M/|z|^{\alpha+1}$ for all $|z|\ge z_1$. Hence, statement (i) holds. 

We next prove statement (ii). By Markov's inequality, one has that 
\begin{align*}
\int_{|z|\ge\tau} p(z)\mathrm{d}z = \mathbb{P}(|\zeta|\ge \tau) = \mathbb{P}(|\zeta|^\alpha\ge \tau^\alpha)  \le \frac{\E[|\zeta|^\alpha]}{\tau^\alpha} \le \frac{M}{\tau^\alpha}.
\end{align*}
Let $z^\prime\ge 2z_1$ be arbitrarily chosen. Notice that $p(\cdot)$ is nonincreasing over $[z_1,\infty)$. It follows from the above inequality that 
\begin{align*}
\frac{z^\prime p(z^\prime)}{2}\le \int_{z^\prime/2}^{z^\prime} p(z)\mathrm{d}z \le M\Big(\frac{2}{z^\prime}\Big)^\alpha.
\end{align*}
Rearranging terms of  this inequality, we obtain that $p(z)\le M(2/|z|)^{\alpha+1}$ holds for any $z\ge 2z_1$. For any $z\le -2z_1$, the same argument applies to $q(\cdot):=p(-\cdot)$. Hence, statement (ii) holds as desired.
\end{proof}

\begin{proof}[\textbf{Proof of Lemma \ref{lem:bias-asym}}]
Since $\zeta$ has mean zero, we have
\begin{align}\label{ppt-mean-0}
\int^\tau_{-\tau} zp(z)\mathrm{d}z = - \int^{\infty}_\tau zp(z)\mathrm{d}z  - \int^{-\tau}_{-\infty} zp(z)\mathrm{d}z.
\end{align}
It then follows that for any $\tau\ge z_1$,
\begin{align*}
\int^{\infty}_\tau zp(z)\mathrm{d}z & \le \int^{\infty}_\tau M z^{-\alpha}\mathrm{d}z = \frac{M}{(\alpha-1)\tau^{\alpha-1}},\\
-\int_{-\infty}^{-\tau} zp(z)\mathrm{d}z & \le -\int_{-\infty}^{-\tau} M (-z)^{-\alpha}\mathrm{d}z = \frac{M}{(\alpha-1)\tau^{\alpha-1}}.
\end{align*}
These along with \eqref{ppt-mean-0} imply the first relation in \eqref{decay-rate-a}. We next prove the second relation in \eqref{decay-rate-a}. Notice that 
\begin{align*}
\Big|\int^\infty_\tau(p(z)-p(-z))\mathrm{d}z\Big| \le \int^\infty_\tau(p(z)+p(-z))\mathrm{d}z \le \int^\infty_\tau 2M z^{-(\alpha+1)}\mathrm{d}z = \frac{2M}{\alpha \tau^\alpha}  
\end{align*}
Hence, the second relation in \eqref{decay-rate-a} holds as desired.
\end{proof}

\subsection{Proof of the Main Results in Section \ref{sec:bias-var}}\label{subsec:pf-bias-var}

In this subsection, we prove Lemma \ref{lem:bias-var-decomp}, and Theorems \ref{thm:trade-off1} and \ref{thm:trade-off2}.
\begin{proof}[\textbf{Proof of Lemma \ref{lem:bias-var-decomp}}]
We first prove \eqref{ineq:1dim-bias}. By the definition of the projection operator, we have
\begin{align}\label{proj-a+v}
\Pi_{[-\tau,\tau]}(a+z) = \left\{\begin{array}{cc}
a+z   &\text{if}\ |a+z|<\tau,\\
\mathrm{sgn}(a+z)\tau &\text{if}\ |a+ z|\ge\tau
\end{array}\right.    \qquad\forall z\in\R.
\end{align}
By splitting the integral into subintervals and rearranging terms, we derive that for all $\tau>0$,
\begin{align}
\E\big[\Pi_{[-\tau,\tau]}(a+\zeta)\big] - a & = \int_{-\infty}^{\infty}\big(\Pi_{[-\tau,\tau]}(a+z) - a\big) p(z)\mathrm{d}z \nonumber\\
& = \int_{-\tau-a}^{\tau-a} zp(z)\mathrm{d}z - (\tau+a) \int_{-\infty}^{-\tau-a}p(z)\mathrm{d}z + (\tau-a)\int_{\tau-a}^\infty p(z)\mathrm{d}z\nonumber\\
& = \int_{-\tau}^{\tau}zp(z)\mathrm{d}z + \int_{-\tau-a}^{-\tau}zp(z)\mathrm{d}z - \int_{\tau-a}^\tau zp(z)\mathrm{d}z\nonumber\\
&\qquad\ - \tau\int_{-\infty}^{-\tau}p(z)\mathrm{d}z + \tau\int_\tau^\infty p(z) \mathrm{d}z + \tau \int_{-\tau-a}^{-\tau} p(z)\mathrm{d}z + \tau\int_{\tau-a}^\tau p(z) \mathrm{d}z\nonumber\\
&\qquad\ - a\Big(\int^{-\tau-a}_{-\infty}p(z)\mathrm{d}z + \int^{\infty}_{\tau-a}p(z)\mathrm{d}z\Big)\nonumber\\
& = \int_{-\tau}^\tau zp(z)\mathrm{d}z + \tau\int_\tau^\infty (p(z) - p(-z)) \mathrm{d}z\nonumber\\
&\qquad\ +  \int_{-\tau-a}^{-\tau}(z+\tau)p(z)\mathrm{d}z - \int_{\tau-a}^{\tau}(z-\tau)p(z)\mathrm{d}z\nonumber\\
&\qquad\ - a\Big(\int^{-\tau-a}_{-\infty}p(z)\mathrm{d}z + \int^{\infty}_{\tau-a}p(z)\mathrm{d}z\Big).\label{pf-bias-decom}
\end{align}
Using the fact that $p(z)\le M_2/|z|^{\alpha+1}$ for all $|z|\ge z_1$, we obtain that for all $\tau\ge z_1+|a|$,
\begin{align}
&\Big|\int_{\tau-a}^{\tau}(z-\tau)p(z)\mathrm{d}z\Big| \le \int_{\tau-a}^{\tau}|(z-\tau)p(z)|\mathrm{d}z \le a^2 \max_{z\in[\tau-|a|,\tau+|a|]}\{p(z)\} \le \frac{a^2M_2}{(\tau-|a|)^{\alpha+1}},\label{decay-1dim-1}\\
&\Big|\int_{-\tau-a}^{-\tau}(z+\tau)p(z)\mathrm{d}z \Big| \le \int_{-\tau-a}^{-\tau}|(z+\tau)p(z)|\mathrm{d}z \le a^2 \max_{z\in[-\tau-|a|,-\tau+|a|]}\{p(z)\} \le \frac{a^2M_2}{(\tau-|a|)^{\alpha+1}}.\label{decay-1dim-2}
\end{align}
In addition, we have that for all $\tau\ge z_1+|a|$,
\begin{align}
&\int^{\infty}_{\tau-a}p(z)\mathrm{d}z\le \int^{\infty}_{\tau-|a|}M_2z^{-(\alpha+1)}\mathrm{d}z \le \frac{M_2}{\alpha(\tau-|a|)^\alpha},\label{decay-1dim-3}\\
&\int^{-\tau-a}_{-\infty}p(z)\mathrm{d}z\le \int^{-\tau+|a|}_{-\infty}M_2(-z)^{-(\alpha+1)}\mathrm{d}z \le \frac{M_2}{\alpha(\tau-|a|)^\alpha}.\label{decay-1dim-4}
\end{align}
By substituting \eqref{decay-1dim-1}, \eqref{decay-1dim-2}, \eqref{decay-1dim-3} and \eqref{decay-1dim-4} into \eqref{pf-bias-decom}, we can obtain that \eqref{ineq:1dim-bias} holds as desired.

We now prove \eqref{ineq:1dim-var}. By using \eqref{proj-a+v}, splitting the integral into subintervals, and rearranging terms, we can derive that for all $\tau>0$,
\begin{align}
\E\big[\big(\Pi_{[-\tau,\tau]}(a+\zeta)- a\big)^2\big] & = \int_{-\infty}^{\infty}\big(\Pi_{[-\tau,\tau]}(a+z) - a\big)^2 p(z)\mathrm{d}z \nonumber\\
& = \int_{-\tau-a}^{\tau-a} z^2p(z)\mathrm{d}z + (\tau+a)^2 \int_{-\infty}^{-\tau-a}p(z)\mathrm{d}z + (\tau-a)^2\int_{\tau-a}^\infty p(z)\mathrm{d}z.\label{var-growth-pf}
\end{align}
Since $\E[|\zeta|^\alpha]\le M_1$, we can see that for all $\tau>0$,
\begin{align*}
\int_{-\tau-a}^{\tau-a} z^2p(z)\mathrm{d}z \le (\tau+|a|)^{2-\alpha} \int_{-\tau-a}^{\tau-a} |z|^\alpha p(z) \mathrm{d}z \le (\tau+|a|)^{2-\alpha} \E[|\zeta|^\alpha] \le M_1 (\tau+|a|)^{2-\alpha}.    
\end{align*}
Combining this, \eqref{decay-1dim-3} and \eqref{decay-1dim-4} with \eqref{var-growth-pf}, we obtain that \eqref{ineq:1dim-var} holds for all $\tau\ge z_1 + |a|$ as desired.
\end{proof}

For convenience, we denote
\begin{align}\label{def:proj-one-dim}
G_{i,\tau}(x;\xi) = \Pi_{[-\tau,\tau]}(G_i(x;\xi)).    
\end{align}
Then, by $G_\tau(x;\xi)=\Pi_{\{g:\|g\|_\infty\le\tau\}}(G(x;\xi))$, one has that
\begin{align*}
G_\tau(x;\xi) = \begin{bmatrix}
G_{1,\tau}(x;\xi)\\
\vdots\\
G_{n,\tau}(x;\xi)
\end{bmatrix}.    
\end{align*}

\begin{proof}[\textbf{Proof of Theorem \ref{thm:trade-off1}}]
In view of Assumption \ref{asp:basic}(c), \eqref{grad-noise}, and \eqref{def:proj-one-dim}, we see that the assumptions of Lemma \ref{lem:bias-var-decomp} holds with $(\zeta,a,p,M_1,M_2,z_1)=(N_i(x,\xi),\nabla_if(x),p_{i,x},\Lambda_1,\Lambda_2,u_1)$. It then follows from Lemma \ref{lem:bias-var-decomp} that for any $x\in\R^n$ and $1\le i\le n$,
\begin{align*}
\big|\E\big[G_{i,\tau}(x;\xi)\big] - \nabla_i f(x)\big| & \le \Big|\int_{-\tau}^\tau up_{i,x}(u)\mathrm{d}u\Big| + \Big|\tau\int_\tau^\infty (p_{i,x}(u) - p_{i,x}(-u)) \mathrm{d}u\Big|\nonumber\\
&\qquad +  \frac{2\Lambda_2 |\nabla_i f(x)|}{(\tau-|\nabla_i f(x)|)^{\alpha}}\Big(\frac{|\nabla_i f(x)|}{\tau - |\nabla_i f(x)|} + \frac{1}{\alpha}\Big),\\
\E\big[\big(G_{i,\tau}(x;\xi)- \nabla_i f(x)\big)^2\big]&\le \Lambda_1 (\tau+|\nabla_i f(x)|)^{2-\alpha} + \frac{2\Lambda_2(\tau^2 + \nabla_i f(x)^2)}{\alpha(\tau-|\nabla_i f(x)|)^\alpha}
\end{align*}    
hold for all $\tau\ge u_1+|\nabla_i f(x)|$. Then, using \eqref{def:Uf-Dh} and the above relations, we obtain that for any $x\in\mathrm{dom}\,h$ and $1\le i\le n$,
\begin{align}
\big|\E\big[G_{i,\tau}(x;\xi)\big] - \nabla_i f(x)\big| & \le \Big|\int_{-\tau}^\tau up_{i,x}(u)\mathrm{d}u\Big| + \Big|\tau\int_\tau^\infty (p_{i,x}(u) - p_{i,x}(-u)) \mathrm{d}u\Big|\nonumber\\
&\qquad +  \frac{2\Lambda_2 U_f}{(\tau- U_f)^{\alpha}}\Big(\frac{U_f}{\tau - U_f} + \frac{1}{\alpha}\Big),\label{ineq:upbd-1dim-bias}\\
\E\big[\big(G_{i,\tau}(x;\xi)- \nabla_i f(x)\big)^2\big]&\le \Lambda_1 (\tau+U_f)^{2-\alpha} + \frac{2\Lambda_2(\tau^2 + U_f^2)}{\alpha(\tau- U_f)^\alpha}\label{ineq:upbd-1dim-var}
\end{align}    
hold for all $\tau\ge u_1+U_f \overset{\eqref{def:tau-underline}}{=} \tau_{(1)}$.

We first prove statement (i). Using \eqref{ineq:upbd-1dim-bias}, we derive that for any $x\in\mathrm{dom}\,h$ and $\tau\ge\tau_{(1)}$,
\begin{align}
\big\|\E\big[G_{\tau}(x;\xi)\big] - \nabla f(x)\big\| & \le \sqrt{n}\big\|\E\big[G_{\tau}(x;\xi)\big] - \nabla f(x)\big\|_\infty = \sqrt{n} \max_{1\le i\le n}\big\{\big|\E\big[G_{i,\tau}(x;\xi)\big] - \nabla_i f(x)\big|\big\} \nonumber\\
&\overset{\eqref{ineq:upbd-1dim-bias}}{\le}\sqrt{n}\Big[\max_{1\le i\le n}\Big\{\Big|\int_{-\tau}^\tau up_{i,x}(u)\mathrm{d}u\Big|\Big\} + \max_{1\le i\le n}\Big\{\Big|\tau\int_\tau^\infty (p_{i,x}(u) - p_{i,x}(-u)) \mathrm{d}u\Big|\Big\}\nonumber\\
&\qquad +  \frac{2\Lambda_2 U_f}{(\tau- U_f)^{\alpha}}\Big(\frac{U_f}{\tau - U_f} + \frac{1}{\alpha}\Big)\Big]\nonumber\\
&\le \sqrt{n}\Big[\max_{1\le i\le n}\Big\{\Big|\int_{-\tau}^\tau up_{i,x}(u)\mathrm{d}u\Big|\Big\} + \max_{1\le i\le n}\Big\{\Big|\tau\int_\tau^\infty (p_{i,x}(u) - p_{i,x}(-u)) \mathrm{d}u\Big|\Big\}\nonumber\\
&\qquad +  \frac{2\Lambda_2 U_f(U_f/u_1 + 1/\alpha)}{(\tau- U_f)^{\alpha}}\Big],\label{upbd-bias-decom}
\end{align}
where the last inequality is due to $\tau\ge\tau_{(1)}\overset{\eqref{def:tau-underline}}{=} u_1+U_f$. In view of this and \eqref{truncate-bias} and \eqref{near-sym}, we see that 
\begin{align*}
\lim_{\tau\to\infty} \sup_{x\in\mathrm{dom}\,h} \big\{\big\|\E\big[G_{\tau}(x;\xi)\big] - \nabla f(x)\big\|\big\} = 0,  
\end{align*}
which along with \eqref{def:T-eps} implies that $\mathcal{T}(\varepsilon)$ is nonempty for any $\varepsilon\in(0,1)$. In addition, it follows from \eqref{ineq:upbd-1dim-var} that for any $x\in\mathrm{dom}\,h$ and $\tau\ge\tau_{(1)}$,
\begin{align*}
\E\big[\big\|G_{\tau}(x;\xi) - \nabla f(x)\big\|^2\big]&=\sum_{i=1}^n \E\big[\big(G_{i,\tau}(x;\xi)- \nabla_i f(x)\big)^2\big]\\
&\overset{\eqref{ineq:upbd-1dim-var}}{\le} n\bigg[\Lambda_1 (\tau+ U_f)^{2-\alpha} + \frac{2\Lambda_2(\tau^2 + U_f^2)}{\alpha(\tau-U_f)^\alpha}\bigg] = \sigma^2(\tau).
\end{align*} 
This immediately implies \eqref{def:sigma-tau}. Hence, statement (i) holds as desired.

We next prove statement (ii). Under Assumption \ref{asp:basic}(c), we see that the assumptions of Lemma \ref{lem:bias-asym} holds with $(\zeta,p,M_2,z_1)=(N_i(x,\xi),p_{i,x},\Lambda_2,u_1)$. It then follows from Lemma \ref{lem:bias-asym} that 
\begin{align*}
\Big|\int_{-\tau}^\tau up_{i,x}(u)\mathrm{d}u\Big| \le \frac{2\Lambda_2}{(\alpha-1)\tau^{\alpha-1}},\quad \Big|\tau\int_\tau^\infty (p_{i,x}(u) - p_{i,x}(-u)) \mathrm{d}u\Big|\le \frac{2\Lambda_2}{\alpha \tau^{\alpha-1}}\qquad\forall x\in\R^n,1\le i\le n.
\end{align*}
By substituting these into \eqref{upbd-bias-decom} and using $\alpha\in(1,2]$ and \eqref{def:tau-underline}, we can derive that for all $\tau\ge \tau_{(1)}$, $x\in\mathrm{dom}\,h$, and $1\le i\le n$,
\begin{align}
\big\|\E\big[G_{\tau}(x;\xi)\big] - \nabla f(x)\big\| & \le \sqrt{n}\Big[\frac{2(2\alpha-1)\Lambda_2}{\alpha(\alpha-1)\tau^{\alpha-1}} +  \frac{2\Lambda_2 U_f(U_f/u_1 + 1/\alpha)}{(\tau- U_f)^{\alpha}}\Big]\nonumber\\
&\le \sqrt{n}\Big[\frac{3\Lambda_2}{(\alpha-1)\tau^{\alpha-1}} +  \frac{2\Lambda_2 U_f(u_1+U_f)}{u_1(\tau- U_f)^{\alpha}}\Big]\nonumber\\
&\overset{\eqref{def:tau-underline}}{=} \sqrt{n}\Big[\frac{3\Lambda_2}{(\alpha-1)\tau^{\alpha-1}} +  \frac{2\Lambda_2 U_f\tau_{(1)}}{u_1(\tau- U_f)^{\alpha}}\Big].\nonumber
\end{align}
By this, one can see that $\|\E[G_{\tau}(x;\xi)] - \nabla f(x)\|\le\varepsilon/2+\varepsilon/2$ holds for $\tau_{1}(\varepsilon)$ defined in \eqref{def:tau1-eps}. Hence, $\tau_{1}(\varepsilon)\in\mathcal{T}(\varepsilon)$ holds, which completes the proof.
\end{proof}

\begin{proof}[\textbf{Proof of Theorem \ref{thm:trade-off2}}]
Notice that the assumptions in Theorem \ref{thm:trade-off1} holds. By the same arguments as for proving \eqref{upbd-bias-decom}, one has that for any $x\in\mathrm{dom}\,h$ and $\tau\ge\tau_{(1)}$,
\begin{align}
\big\|\E\big[G_{\tau}(x;\xi)\big] - \nabla f(x)\big\| & \le \sqrt{n}\Big[\max_{1\le i\le n}\Big\{\Big|\int_{-\tau}^\tau up_{i,x}(u)\mathrm{d}u\Big|\Big\} + \max_{1\le i\le n}\Big\{\Big|\tau\int_\tau^\infty (p_{i,x}(u) - p_{i,x}(-u)) \mathrm{d}u\Big|\Big\}\nonumber\\
&\qquad +  \frac{2\Lambda_2 U_f(U_f/u_1 + 1/\alpha)}{(\tau- U_f)^{\alpha}}\Big]\label{upbd-bias-decom2}
\end{align}
In addition, recall from Assumption \ref{asp:decay} that
\begin{align*}
\Big|\int_{-\tau}^\tau up_{i,x}(u)\mathrm{d}u\Big| \le \Gamma_1 \tau^{-\alpha},\quad \Big|\tau\int_\tau^\infty (p_{i,x}(u) - p_{i,x}(-u)) \mathrm{d}u\Big|\le \Gamma_2\tau^{-\alpha}\qquad\forall x\in\R^n,\tau\ge\tau_{(2)},1\le i\le n.
\end{align*}
By substituting these into \eqref{upbd-bias-decom2}, we obtain that for all $\tau\ge \tau_{(2)}$, $x\in\mathrm{dom}\,h$, and $1\le i\le n$,
\begin{align*}
\big\|\E\big[G_{\tau}(x;\xi)\big] - \nabla f(x)\big\| \le  \sqrt{n}\Big[(\Gamma_1+\Gamma_2)\tau^{-\alpha}  +  \frac{2\Lambda_2 U_f(U_f/u_1 + 1/\alpha)}{(\tau- U_f)^{\alpha}}\Big].
\end{align*}
By this, one can see that $\|\E[G_{\tau}(x;\xi)] - \nabla f(x)\|\le\varepsilon/2+\varepsilon/2$ holds for $\tau_{2}(\varepsilon)$ defined in \eqref{def:tau-eps-2}. Hence, the conclusion of this theorem holds as desired.
\end{proof}

\subsection{Proof of the Main Results in Section \ref{subsec:complexity-cvx}}\label{subsec:pf-complexity-cvx}

In this subsection, we prove Lemmas \ref{lem:upbd-obj-gap} and \ref{lem:upbd-obj-gap-scvx}, and Theorems \ref{thm:cmplex-cvx} and \ref{thm:cmplex-scvx}.

\begin{proof}[\textbf{Proof of Lemma \ref{lem:upbd-obj-gap}}]
Fix any $k\ge0$. By the optimality condition of \eqref{prox-step} with $\eta_k=\eta$, it follows that there exists $h^\prime(x^{k+1})\in\partial h(x^{k+1})$ such that
\begin{align*}
G_\tau(x^k;\xi_k) + \frac{1}{\eta}(x^{k+1} - x^k) + h^\prime(x^{k+1}) = 0,
\end{align*}
which along with the convexity of $h$ implies that
\begin{align}
h(x^{k+1}) & \le h(x^*) + h^\prime(x^{k+1})^T(x^{k+1} - x^*) = h(x^*) + G_{\tau}(x^k;\xi_k)^T(x^* - x^{k+1}) + \frac{1}{\eta}(x^{k+1} - x^k)^T(x^* - x^{k+1}) \nonumber\\
& = h(x^*) + G_{\tau}(x^k;\xi_k)^T(x^* - x^{k+1}) + \frac{1}{2\eta}(\|x^k-x^*\|^2 - \|x^{k+1} - x^*\|^2 - \|x^{k+1} - x^k\|^2).\label{cvx-h-ineq}
\end{align}
By \eqref{ineq:descent}, the definition of $D_h$ in \eqref{def:Uf-Dh}, and the convexity of $f$, one has
\begin{align}
f(x^{k+1}) & \le f(x^k) + \nabla f(x^k)^T(x^{k+1} - x^k) + \frac{L_f}{2}\|x^{k+1} - x^k\|^2\nonumber\\
& \le f(x^*) + \nabla f(x^k)^T(x^{k+1} - x^*) + \frac{L_fD_h}{2}\|x^{k+1} - x^k\|.\label{cvx-f-ineq}
\end{align}
Combining this with \eqref{cvx-h-ineq}, we obtain that
\begin{align}
F(x^{k+1}) & \le F(x^*) + (\nabla f(x^k) - G_{\tau}(x^k;\xi_k))^T(x^k - x^*) + (\nabla f(x^k) - G_{\tau}(x^k;\xi_k))^T(x^{k+1} - x^k) \nonumber \\
&\qquad + \frac{1}{2\eta}(\|x^k-x^*\|^2 - \|x^{k+1} - x^*\|^2) - \frac{1}{2\eta}\|x^{k+1} - x^k\|^2 + \frac{L_fD_h}{2}\|x^{k+1} - x^k\| \nonumber\\
&\le F(x^*) + (\nabla f(x^k) - G_{\tau}(x^k;\xi_k))^T(x^k - x^*) + (\|\nabla f(x^k) - G_{\tau}(x^k;\xi_k)\|^2 + L_f^2D_h^2) \eta \nonumber\\
&\qquad + \frac{1}{2\eta}(\|x^k-x^*\|^2 - \|x^{k+1} - x^*\|^2), \label{desc-pf1-Fk+1}
\end{align}
where the last inequality is due to 
\begin{align*}
\frac{L_fD_h}{2}\|x^{k+1} - x^k\| &\le \frac{\|x^{k+1} - x^k\|^2}{4\eta} + \frac{L_f^2D_h^2\eta}{4},\\
(\nabla f(x^k) - G_{\tau}(x^k;\xi_k))^T(x^{k+1} - x^k) & \le \frac{\|x^{k+1} - x^k\|^2}{4\eta} + \|\nabla f(x^k) - G_{\tau}(x^k;\xi_k)\|^2\eta.    
\end{align*}
In addition, we recall from Theorem \ref{thm:trade-off1}(i) that for all $\tau\ge \tau_{(1)}$,
\begin{align}\label{pf1-bd-Var}
\E_{\xi_k}[\|\nabla f(x^k) - G_{\tau}(x^k;\xi_k)\|^2]\le\sigma^2(\tau).    
\end{align}
Using the definition of $D_h$ in \eqref{def:Uf-Dh}, and the definition of $\Delta(\cdot)$ in \eqref{def:Delta-tau}, we obtain that for all $\tau\ge0$,
\begin{align}\label{pf1-bd-bias}
\E_{\xi_k}[(\nabla f(x^k) - G_{\tau}(x^k;\xi_k))^T(x^k - x^*)]\le \|\nabla f(x^k) - \E_{\xi_k}[G_{\tau}(x^k;\xi_k)]\|\cdot\|x^k-x^*\| \le D_h\Delta(\tau).    
\end{align}
Taking expectation on \eqref{desc-pf1-Fk+1} with respect to $\xi_k$, using \eqref{pf1-bd-Var} and \eqref{pf1-bd-bias}, and rearranging terms, we obtain that \eqref{ineq:descent-cvx} holds as desired.
\end{proof}

\begin{proof}[\textbf{Proof of Theorem \ref{thm:cmplex-cvx}}]
Notice from Theorem \ref{thm:trade-off1}(i) that $\mathcal{T}\big(\frac{\epsilon}{2D_h}\big)\neq\emptyset$ under Assumption \ref{asp:basic}(c). Thus, $\tau_\epsilon$ exists. Using \eqref{alg1-ave-step}, \eqref{ineq:descent-cvx} with $(\eta,\tau)=(\eta_\epsilon,\tau_\epsilon)$, and the convexity of $F$, we obtain that for all $K\ge1$,
\begin{align}
\E[F(z^K)] - F^* \overset{\eqref{alg1-ave-step}}{\le} \frac{1}{K}\sum_{k=0}^{K-1}(\E[F(x^{k+1})] - F^*)  \overset{\eqref{ineq:descent-cvx}}{\le} \frac{\|x^0 - x^*\|^2}{2K\eta_\epsilon} + D_h\Delta(\tau_\epsilon) + \bigg(\frac{L_f^2D_h^2}{4} + \sigma^2(\tau_\epsilon)\bigg)\eta_\epsilon.\label{thm1-pf-ineq1}
\end{align}
Recall from the definition of $D_h$ in \eqref{def:Uf-Dh} that $\|x^0-x^*\|\le D_h$. In addition, by the definitions of $\mathcal{T}(\cdot)$ and $\tau_\epsilon$ in \eqref{def:T-eps} and \eqref{def:tau-eta-cvx}, respectively, one has that $D_h\Delta(\tau_\epsilon)\le\epsilon/2$. Combining these with \eqref{thm1-pf-ineq1}, we obtain that for all $K\ge1$,
\begin{align*}
\E[F(z^K)] - F^* & \overset{\eqref{thm1-pf-ineq1}}{\le} \frac{D_h^2}{2K\eta_\epsilon} + \bigg(\frac{L_f^2D_h^2}{4} + \sigma^2(\tau_\epsilon)\bigg)\eta_\epsilon + \frac{\epsilon}{2}  = \min_{\hat{\eta}} \bigg\{\frac{D_h^2}{2K\hat\eta} + \bigg(\frac{L_f^2D_h^2}{4} + \sigma^2(\tau_\epsilon)\bigg)\hat\eta\bigg\}  + \frac{\epsilon}{2}\nonumber\\
& = \sqrt{2}D_h\bigg(\frac{L_f^2D_h^2/4 + \sigma^2(\tau_\epsilon)}{K}\bigg)^{1/2} + \frac{\epsilon}{2},
\end{align*}
where the first equality is due to the definition of $\eta_\epsilon$ in \eqref{def:tau-eta-cvx}. Then, by this, one can see that $\E[F(z^K)] - F^*\le\epsilon/2+\epsilon/2$ holds for all $K$ satisfying \eqref{cvx-K-iter-cmplx}. Hence, the conclusion of this theorem holds as desired.
\end{proof}

The following inequality provides an estimation of the $K$th harmonic number:
\begin{align}\label{sum-k+1}
\sum_{k=0}^{K-1}\frac{1}{k+1} \le \sum_{k=0}^{K-1} \int_{k+1/2}^{k+3/2} \frac{1}{t+1}\mathrm{d}t = \int_{1/2}^{K+1/2} \frac{1}{t+1}\mathrm{d}t = \ln(2K+1),
\end{align}
where the first inequality follows from the convexity of $\phi(t)=1/(t+1)$ with $t\ge0$ (see also \cite[Lemma 2]{he2025complexity} with $(a,b,\beta)=(1,K,1)$).

We next provide a lemma that will be used to derive complexity bounds for Algorithm \ref{alg:c-spgm}. Its proof follows similarly to that of \cite[Lemma 3]{he2025complexity}.

\begin{lemma}\label{lem:rate-complexity}
Let $u\in(0,e^{-1})$ be given. Then, $v^{-1}\ln v\le 2u$ holds for all $v\ge u^{-1}\ln(1/u)$.
\end{lemma}

\begin{proof}
    Fix any $v$ satisfying $v\ge u^{-1}\ln(1/u)$. Then, by $u\in(0,e^{-1})$ and the fact that $\psi(u)=u^{-1}\ln(1/u)$ is decreasing, one has $v\ge u^{-1}\ln(1/u) > e$. Let $\phi(t)=t^{-1}\ln t$. It can be verified that $\phi$ is decreasing on $(e,\infty)$. By this and $v\ge u^{-1}\ln(1/u)>e$, one has that 
    \begin{align*}
        v^{-1}\ln v = \phi(v) \le \phi(u^{-1}\ln(1/u)) = u \Big(1 + \frac{\ln\ln(1/u)}{\ln(1/u)}\Big) \le 2u,
    \end{align*}
    where the last inequality follows from $\ln\ln(1/u)\le\ln(1/u)$ due to $u\in(0,e^{-1})$. Hence, the conclusion of this lemma holds.
\end{proof}

\begin{proof}[\textbf{Proof of Lemma \ref{lem:upbd-obj-gap-scvx}}]
By similar arguments for proving \eqref{cvx-h-ineq} and \eqref{cvx-f-ineq}, one can prove that
\begin{align*}
h(x^{k+1}) & \le h(x^*) + G_{\tau}(x^k;\xi_k)^T(x^* - x^{k+1}) + \frac{1}{2\eta_k}(\|x^k-x^*\|^2 - \|x^{k+1} - x^*\|^2 - \|x^{k+1} - x^k\|^2),\\
f(x^{k+1}) & \le f(x^*) + \nabla f(x^k)^T(x^{k+1} - x^*) - \frac{\mu_f}{2}\|x^k - x^*\|^2 + \frac{L_fD_h}{2}\|x^{k+1} - x^k\|.
\end{align*}
By the same arguments for proving \eqref{desc-pf1-Fk+1}, one has that
\begin{align}
F(x^{k+1}) & \le F(x^*) + (\nabla f(x^k) - G_{\tau}(x^k;\xi_k))^T(x^k - x^*) + \bigg(\|\nabla f(x^k) - G_{\tau}(x^k;\xi_k)\|^2 + \frac{L_f^2D_h^2}{4}\bigg) \eta_k \nonumber\\
&\qquad + \Big(\frac{1}{2\eta_k}-\frac{\mu_f}{2}\Big)\|x^k-x^*\|^2 - \frac{1}{2\eta_k}\|x^{k+1} - x^*\|^2.\label{2-desc-pf1-Fk+1}
\end{align}
Using the definition of $\Delta(\cdot)$ in \eqref{def:Delta-tau}, we obtain that for all $\tau\ge0$,
\begin{align}
\E_{\xi_k}[(\nabla f(x^k) - G_{\tau}(x^k;\xi_k))^T(x^k - x^*)] & \le \|\nabla f(x^k) - \E_{\xi_k}[G_{\tau}(x^k;\xi_k)]\| \cdot \|x^k-x^*\| \nonumber\\
& \le \frac{\mu_f}{4}\|x^k-x^*\|^2 + \frac{\Delta^2(\tau)}{\mu_f}.    \label{pf2-bd-bias}
\end{align}
Taking expectation on \eqref{2-desc-pf1-Fk+1} with respect to $\xi_k$, using \eqref{pf1-bd-Var} and \eqref{pf2-bd-bias}, and rearranging terms, we obtain that \eqref{ineq:descent-scvx} holds as desired.
\end{proof}

\begin{proof}[\textbf{Proof of Theorem \ref{thm:cmplex-scvx}}]
Note from Theorem \ref{thm:trade-off1}(i) that $\mathcal{T}\Big(\sqrt{\frac{\mu_f\epsilon}{2}}\Big)\neq\emptyset$ under Assumption \ref{asp:basic}(c). Thus, $\tilde{\tau}_\epsilon$ exists. Using \eqref{alg1-ave-step}, \eqref{ineq:descent-scvx} with $\eta_k=\tilde{\eta}_k$ for all $k\ge0$ and $\tau=\tilde\tau_\epsilon$, and the convexity of $F$, we obtain that for all $K\ge1$,
\begin{align}
\E[F(z^K)] - F^* & \overset{\eqref{alg1-ave-step}}{\le} \frac{1}{K}\sum_{k=0}^{K-1}(\E[F(x^{k+1})] - F^*) \nonumber\\
&\overset{\eqref{ineq:descent-scvx}}{\le} \frac{1}{K}\sum_{k=0}^{K-1}\Big[\frac{\mu_f k}{4}\|x^k-x^*\|^2 - \frac{\mu_f(k+1)}{4}\|x^{k+1} - x^*\|^2\Big] + \frac{L_f^2D_h^2/2 + 2\sigma^2(\tilde\tau_\epsilon)}{\mu_fK}\sum_{k=0}^{K-1}\frac{1}{k+1}\nonumber\\
&\qquad  + \frac{\Delta^2(\tilde\tau_\epsilon)}{\mu_f}\nonumber\\
& = \frac{L_f^2D_h^2/2 + 2\sigma^2(\tilde\tau_\epsilon)}{\mu_fK}\sum_{k=0}^{K-1}\frac{1}{k+1} + \frac{\Delta^2(\tilde\tau_\epsilon)}{\mu_f},\label{ineq:pf2-scvx-1}
\end{align}
where the second inequality is due to \eqref{ineq:descent-scvx} and the definition of $\tilde{\eta}_k$ in \eqref{def:tilde-tau-eta-scvx}. By the definitions of $\mathcal{T}(\cdot)$ and $\tilde\tau_\epsilon$ in \eqref{def:T-eps} and \eqref{def:tilde-tau-eta-scvx}, respectively, one has that $\Delta^2(\tilde\tau_\epsilon)/\mu_f\le\epsilon/2$. Then, using this, \eqref{sum-k+1} and \eqref{ineq:pf2-scvx-1}, we obtain that for all $K\ge3$,
\begin{align}\label{ineq:last2-bd-EFz-Fs}
\E[F(z^K)] - F^* \le \frac{(L_f^2D_h^2/2 + 2\sigma^2(\tilde\tau_\epsilon))\ln(2K+1)}{\mu_fK} + \frac{\epsilon}{2} \le \frac{(L_f^2D_h^2 + 4\sigma^2(\tilde\tau_\epsilon))\ln K}{\mu_fK} + \frac{\epsilon}{2}, 
\end{align}
where the second inequality is because $\ln(2K+1)\le 2\ln K$ for all $K\ge3$. In addition, using Lemma \ref{lem:rate-complexity} with $(v,u)=(K,\frac{\mu_f\epsilon}{4(L_f^2D_h^2 + 4\sigma^2(\tilde{\tau}_\epsilon))})$, we obtain that
\begin{align*}
\frac{\ln K}{K} \le \frac{\mu_f\epsilon}{2(L_f^2D_h^2 + 4\sigma^2(\tilde{\tau}_\epsilon))}\qquad\forall K\ge \bigg(\frac{4(L_f^2D_h^2 + 4\sigma^2(\tilde{\tau}_\epsilon))}{\mu_f\epsilon}\bigg) \ln\bigg(\frac{4(L_f^2D_h^2 + 4\sigma^2(\tilde{\tau}_\epsilon))}{\mu_f\epsilon}\bigg),    
\end{align*}
which along with \eqref{ineq:last2-bd-EFz-Fs} implies that $\E[F(z^K)] - F^* \le\epsilon/2 + \epsilon/2$ holds for all $K$ satisfying \eqref{scvx-K-iter-cmplx}. Hence, the conclusion of this theorem holds as desired.
\end{proof}

\subsection{Proof of the Main Results in Section \ref{subsec:complexity-ncvx}}\label{subsec:pf-complexity-ncvx}

In this subsection, we prove Lemma \ref{lem:pot-decrease} and Theorem \ref{thm:cmplex-ncvx}.

\begin{lemma}
Suppose that Assumption \ref{asp:basic} holds. Let $L_f$ be given in Assumption~\ref{asp:basic}, and $\tau_{(1)}$, $\Delta(\cdot)$ and $\sigma^2(\cdot)$ be defined in \eqref{def:tau-underline}, \eqref{def:Delta-tau} and \eqref{def:sigma-tau}, respectively. Let $\{(x^k,m^k)\}$ be the sequence generated by Algorithm \ref{alg:c-spgm-m} with input parameters $(\eta,\theta)$ and $\{\tau_k\}\subset[\tau_{(1)},\infty)$. Then we have for all $k\ge0$,
\begin{align}\label{ineq:rec-rela-pm}
\E_{\xi_{k+1}}[\|m^{k+1} - \nabla f(x^{k+1})\|^2] \le (1-\theta) \|m^k - \nabla f(x^k)\|^2 + \frac{2L_f^2}{\theta}\|x^{k+1} - x^k\|^2 + 2\theta \Delta^2({\tau_{k+1}}) + \theta^2 \sigma^2({\tau_{k+1}}).
\end{align}
\end{lemma}

\begin{proof}
Fix any $k\ge0$. It follows from \eqref{alg2:momentum-step} that
\begin{align}
&\E_{\xi_{k+1}}[\|m^{k+1} - \nabla f(x^{k+1})\|^2] \overset{\eqref{alg2:momentum-step}}{=} \E_{\xi_{k+1}}[\|(1-\theta)(m^k - \nabla f(x^{k+1})) + \theta (G_{\tau_{k+1}}(x^{k+1};\xi_{k+1}) - \nabla f(x^{k+1}))\|^2]\nonumber\\
& =(1-\theta)^2 \|m^k - \nabla f(x^{k+1})\|^2 + 2(1-\theta)\theta(m^k - \nabla f(x^{k+1}))^T(\E_{\xi_{k+1}}[G_{\tau_{k+1}}(x^{k+1};\xi_{k+1})] - \nabla f(x^{k+1}))\nonumber\\
&\qquad + \theta^2 \E_{\xi_{k+1}}[\|G_{\tau_{k+1}}(x^{k+1};\xi_{k+1}) - \nabla f(x^{k+1})\|^2] \nonumber\\
&\le (1-\theta)\Big(1 - \frac{\theta}{2}\Big)\|m^k - \nabla f(x^{k+1})\|^2 + 2\theta\|\E_{\xi_{k+1}}[G_{\tau_{k+1}}(x^{k+1};\xi_{k+1})] - \nabla f(x^{k+1})\|^2\nonumber\\
&\qquad + \theta^2 \E_{\xi_{k+1}}[\|G_{\tau_{k+1}}(x^{k+1};\xi_{k+1}) - \nabla f(x^{k+1})\|^2],\label{lem1-ineq1-ncvx}
\end{align}
where the last relation is due to $\theta\in(0,1]$ and 
\begin{align*}
&(m^k - \nabla f(x^{k+1}))^T(\E_{\xi_{k+1}}[G_{\tau_{k+1}}(x^{k+1};\xi_{k+1})] - \nabla f(x^{k+1})) \\
&\le \frac{1}{4}\|m^k - \nabla f(x^{k+1})\|^2 + \|\E_{\xi_{k+1}}[G_{\tau_{k+1}}(x^{k+1};\xi_{k+1})] - \nabla f(x^{k+1})\|^2.     
\end{align*}
In addition, we recall from \eqref{def:Delta-tau} and Theorem \ref{thm:trade-off1}(i) that for all $\tau\ge u_1 + U_f$,
\begin{align*}
\|\E_{\xi_{k+1}}[G_{\tau_{k+1}}(x^{k+1};\xi_{k+1})] - \nabla f(x^{k+1})\| \le \Delta(\tau_{k+1}),\quad \E_{\xi_{k+1}}[\|G_\tau(x^{k+1};\xi_{k+1}) - \nabla f(x^{k+1})\|^2] \le \sigma^2(\tau_{k+1}).  
\end{align*}
By substituting them into \eqref{lem1-ineq1-ncvx}, one can derive that for all $c>0$,
\begin{align*}
\E_{\xi_{k+1}}[\|m^{k+1} - \nabla f(x^{k+1})\|^2] & \le (1-\theta)\Big(1 - \frac{\theta}{2}\Big)\|m^k - \nabla f(x^{k+1})\|^2 + 2\theta \Delta^2({\tau_{k+1}}) + \theta^2 \sigma^2({\tau_{k+1}})\\
&\le (1-\theta)\Big(1 - \frac{\theta}{2}\Big)(1+c)\|m^k - \nabla f(x^k)\|^2\\
&\qquad + (1-\theta)\Big(1 - \frac{\theta}{2}\Big)\Big(1+\frac{1}{c}\Big) \|\nabla f(x^{k+1}) - \nabla f(x^k)\|^2\\
&\qquad + 2\theta \Delta^2({\tau_{k+1}}) + \theta^2 \sigma^2({\tau_{k+1}}), 
\end{align*}
where the second inequality follows from $\|a+b\|^2\le (1+c)\|a\|^2 + (1+1/c)\|b\|^2$ for all $c>0$ and $a,b\in\R^n$. Letting $c=\theta/(2-\theta)$, and using Assumption \ref{asp:basic}(b) and the fact that $\theta\in(0,1]$, we obtain that
\begin{align*}
\E_{\xi_{k+1}}[\|m^{k+1} - \nabla f(x^{k+1})\|^2] \le (1-\theta) \|m^k - \nabla f(x^k)\|^2 + \frac{2L_f^2}{\theta}\|x^{k+1} - x^k\|^2 + 2\theta \Delta^2({\tau_{k+1}}) + \theta^2 \sigma^2({\tau_{k+1}}).
\end{align*}
Hence, \eqref{ineq:rec-rela-pm} holds as desired.
\end{proof}

\begin{lemma}
Suppose that Assumption \ref{asp:basic} holds. Let $L_f$ be given in Assumption~\ref{asp:basic}, and $\tau_{(1)}$, $\Delta(\cdot)$ and $\sigma^2(\cdot)$ be defined in \eqref{def:tau-underline}, \eqref{def:Delta-tau} and \eqref{def:sigma-tau}, respectively. Let $\{(x^k,m^k)\}$ be the sequence generated by Algorithm \ref{alg:c-spgm-m} with input parameters $(\eta,\theta)$ satisfying $\eta\in(0,\frac{1}{4L_f}]$ and $\{\tau_k\}\subset[\tau_{(1)},\infty)$. Then we have for all $k\ge0$,
\begin{align}
F(x^{k+1}) & \le F(x^k) - \frac{3}{4\eta}\|x^{k+1} - x^k\|^2 + 2\eta\|\nabla f(x^k) -  m^k\|^2, \label{ineq:descent-ncvx}\\
\|x^{k+1} - x^k\|^2 & \ge \frac{\eta^2}{4}\mathrm{dist}^2(0,\partial F(x^{k+1})) - \frac{3\eta^2}{4}\|\nabla f(x^k) - m^k\|^2.\label{ineq:bd-stat-cond-ncvx}
\end{align}
\end{lemma}

\begin{proof}
    Fix any $k\ge0$. By the optimality condition of \eqref{prox-step-ncvx}, there exists $h^\prime(x^{k+1})\in\partial h(x^{k+1})$ such that
    \begin{align}\label{stat-cond-ncvx}
    m^k + \frac{1}{\eta}(x^{k+1} - x^k) + h^\prime(x^{k+1}) = 0.      
    \end{align}

    We first prove \eqref{ineq:descent-ncvx}. By \eqref{stat-cond-ncvx} and the convexity of $h$, one has that
    \begin{align}\label{ineq:ncvx-h-cvx}
    h(x^{k+1}) \le h(x^k) + h^\prime(x^{k+1})^T(x^{k+1} - x^k) = h(x^k) - (m^k)^T(x^{k+1} - x^k) - \frac{1}{\eta} \|x^{k+1} - x^k\|^2
    \end{align}
    Using \eqref{ineq:descent} with $(y,x)=(x^{k+1}, x^k)$, we obtain that
    \begin{align*}
        f(x^{k+1}) \le f(x^k) + \nabla f(x^k)^T(x^{k+1} - x^k) + \frac{L_f}{2} \|x^{k+1} - x^k\|^2.
    \end{align*}
    Combining this with \eqref{ineq:ncvx-h-cvx}, we obtain that 
    \begin{align*}
        F(x^{k+1}) & \le F(x^k)  + (\nabla f(x^k) - m^k)^T(x^{k+1} - x^k) - \Big(\frac{1}{\eta} - \frac{L_f}{2}\Big)\|x^{k+1} - x^k\|^2 \\
        & \le F(x^k) - \Big(\frac{7}{8\eta} - \frac{L_f}{2}\Big)\|x^{k+1} - x^k\|^2 + 2\eta\|\nabla f(x^k) -  m^k\|^2,
    \end{align*}
    where the second inequality is due to $(\nabla f(x^k) - m^k)^T(x^{k+1} - x^k)\le 2\eta\|\nabla f(x^k) - m^k\|^2 + \|x^{k+1} - x^k\|^2/(8\eta)$. This along with $\eta\in(0,\frac{1}{4L_f}]$ implies that \eqref{ineq:descent-ncvx} holds.

    We next prove \eqref{ineq:bd-stat-cond-ncvx}. It follows from \eqref{stat-cond-ncvx} that
    \begin{align*}
    \mathrm{dist}^2(0,\partial F(x^{k+1})) & \le \|\nabla f(x^{k+1}) + h^\prime(x^{k+1})\|^2 = \Big\|\nabla f(x^{k+1}) - m^k - \frac{1}{\eta}(x^{k+1} - x^k)\Big\|^2 \\
    & = \Big\|\nabla f(x^{k+1}) - \nabla f(x^k) + \nabla f(x^k) - m^k - \frac{1}{\eta}(x^{k+1} - x^k)\Big\|^2 \\
    &\le 3\|\nabla f(x^{k+1}) - \nabla f(x^k)\|^2 + 3\|\nabla f(x^k) - m^k\|^2 + \frac{3}{\eta^2}\|x^{k+1} - x^k\|^2 \\
    &\le 3\Big(L_f^2 + \frac{1}{\eta^2}\Big)\|x^{k+1} - x^k\|^2 + 3\|\nabla f(x^k) - m^k\|^2, 
    \end{align*}
    where the second inequality is due to $\|a+b+c\|^2\le3\|a\|^2 + 3\|b\|^2 + 3\|c\|^2$ for all $a,b,c\in\R^n$, and the last inequality follows from Assumption \ref{asp:basic}(b). This together with $\eta\in(0,\frac{1}{4L_f}]$ implies that \eqref{ineq:bd-stat-cond-ncvx} holds, which completes the proof.
\end{proof}

\begin{proof}[\textbf{Proof of Lemma \ref{lem:pot-decrease}}]
Fix any $k\ge0$. Combining \eqref{def:pot-seq}, \eqref{ineq:rec-rela-pm}, \eqref{ineq:descent-ncvx}, and \eqref{ineq:bd-stat-cond-ncvx}, we obtain that
\begin{align}
\E_{\xi_{k+1}} [\mathcal{P}_{k+1}] & \overset{\eqref{def:pot-seq}}{=} \E_{\xi_{k+1}}\bigg[F(x^{k+1}) + \frac{1}{L_f}\|m^{k+1} - \nabla f(x^{k+1})\|^2\bigg]\nonumber\\
&\overset{\eqref{ineq:rec-rela-pm}\eqref{ineq:descent-ncvx}}{\le} F(x^k) - \frac{3}{4\eta}\|x^{k+1} - x^k\|^2 + 2\eta\|\nabla f(x^k) -  m^k\|^2\nonumber\\
&\qquad + \frac{1-\theta}{L_f} \|m^k - \nabla f(x^k)\|^2 + \frac{2L_f}{\theta}\|x^{k+1} - x^k\|^2 + \frac{2\theta \Delta^2(\tau_{k+1}) + \theta^2 \sigma^2(\tau_{k+1})}{L_f} \nonumber\\
& \overset{\eqref{def:pot-seq}}{=} \mathcal{P}_k - \Big(\frac{3}{4\eta} - \frac{2L_f}{\theta}\Big)\|x^{k+1} - x^k\|^2 + \Big(2\eta - \frac{\theta}{L_f}\Big)\|m^k - \nabla f(x^k)\|^2 + \frac{2\theta \Delta^2(\tau_{k+1}) + \theta^2 \sigma^2(\tau_{k+1})}{L_f} \nonumber\\
&\overset{\eqref{ineq:bd-stat-cond-ncvx}}{\le} \mathcal{P}_k - \Big(\frac{3}{4\eta} - \frac{2L_f}{\theta}\Big)\frac{\eta^2}{4}\mathrm{dist}^2(0,\partial F(x^{k+1})) + \Big[2\eta - \frac{\theta}{L_f} + \Big(\frac{3}{4\eta} - \frac{2L_f}{\theta}\Big)\frac{3\eta^2}{4}\Big]\|m^k - \nabla f(x^k)\|^2\nonumber\\
&\qquad + \frac{2\theta \Delta^2(\tau_{k+1}) + \theta^2 \sigma^2(\tau_{k+1})}{L_f}\nonumber\\
&= \mathcal{P}_k - \frac{\eta}{16}\mathrm{dist}^2(0,\partial F(x^{k+1})) -\eta\|m^k - \nabla f(x^k)\|^2 + 8\eta\Delta^2(\tau_{k+1}) + 16L_f\eta^2 \sigma^2(\tau_{k+1}), \nonumber
\end{align}
where the last relation is due to $\theta=4L_f\eta$. Hence, the conclusion of this lemma holds as desired.
\end{proof}

The next lemma will be used to derive the complexity bounds for Algorithm \ref{alg:c-spgm-m}. Its proof can be found in \cite[Lemma 2]{he2025accelerated} and is therefore omitted here.

\begin{lemma}\label{lem:tech-qm}
Let $a,b,c>0$ be given, $t^*=\min\{c,(a/b)^{1/2}\}$, and $\varphi(t)=a/t + bt$ for $t\in(0,\infty)$. Then, it holds that 
\begin{equation} \label{ineq:tech-qm}
\min_{t\in(0,c]} \varphi(t)=\varphi(t^*)\le a/c + 2(ab)^{1/2}.
\end{equation}
\end{lemma}

\begin{proof}[\textbf{Proof of Theorem \ref{thm:cmplex-ncvx}}]
Notice from Theorem \ref{thm:trade-off1}(i) that $\mathcal{T}(\frac{\epsilon}{32})\neq\emptyset$ under Assumption \ref{asp:basic}(c). Thus, $\hat{\tau}_\epsilon$ exists. By the definition of $(\eta,\theta)$ and $\{\tau_k\}$, we obtain that the assumptions of Lemma \ref{lem:pot-decrease} hold. Then, by \eqref{def:Uf-Dh}, \eqref{ineq:upbd-var-lem}, and \eqref{def:pot-seq}, and $\tau_0=\tau_{(1)}$, one has that
\begin{align}
\E[\mathcal{P}_0] & = F(x^0) + \frac{1}{L_f}\E[\|G_{\tau_0}(x^0;\xi_0) - \nabla f(x^0)\|^2] \le F(x^0) + \frac{\sigma^2(\tau_{(1)})}{L_f},\label{upbd-P0}\\
\E[\mathcal{P}_K] & = F(x^K) + \frac{1}{L_f}\E[\|G_{\tau_K}(x^K;\xi_K) - \nabla f(x^K)\|^2] \ge F(x^K) \ge F_{\mathrm{low}}.\label{lowbd-PK}
\end{align}
Taking expectation on both sides of \eqref{ineq:desc-pot-ncvx} with respect to $\{\xi_i\}_{i=0}^{k+1}$, we have 
\begin{align*}
\E[\mathcal{P}_{k+1}] \le \E[\mathcal{P}_k] - \frac{\eta}{16}\mathrm{dist}^2(0,\partial F(x^{k+1})) + 8\eta\Delta^2(\tau_{k+1}) + 16L_f\eta^2 \sigma^2(\tau_{k+1})\qquad\forall k\ge0.
\end{align*}
Summing up this inequality over $k=0,\ldots,K-1$, and using \eqref{upbd-P0} and \eqref{lowbd-PK}, we obtain that for all $K\ge1$,
\begin{align*}
F_{\mathrm{low}} & \overset{\eqref{lowbd-PK}}{\le} \E[\mathcal{P}_K] \\
& \le \E[\mathcal{P}_0] - \frac{\eta}{16}\sum_{k=0}^{K-1}\mathrm{dist}^2(0,\partial F(x^{k+1})) + 8\eta\sum_{k=0}^{K-1}\Delta^2(\tau_{k+1}) + 16L_f\eta^2 \sum_{k=0}^{K-1}\sigma^2(\tau_{k+1})\\
&\overset{\eqref{upbd-P0}}{\le} F(x^0) + \frac{\sigma^2(\tau_{(1)})}{L_f} - \frac{\eta}{16}\sum_{k=0}^{K-1}\mathrm{dist}^2(0,\partial F(x^{k+1})) + 8\eta\sum_{k=0}^{K-1}\Delta^2(\tau_{k+1}) + 16L_f\eta^2 \sum_{k=0}^{K-1}\sigma^2(\tau_{k+1}).
\end{align*}
By the definitions of $\mathcal{T}(\cdot)$ and $\hat\tau_\epsilon$ in \eqref{def:T-eps} and \eqref{def:hat-eta-theta-tau}, respectively, one has that $\Delta(\hat\tau_\epsilon)\le\epsilon/32$. Rearranging the terms in this inequality and substituting $\eta=\hat{\eta}_\epsilon$ and $\tau_k=\hat{\tau}_\epsilon$ for all $k\ge1$, we obtain that
\begin{align*}
\frac{1}{K}\sum_{k=0}^{K-1}\E[\mathrm{dist}^2(0,\partial F(x^{k+1}))] & \le \frac{16(F(x^0) -F_{\mathrm{low}} + \sigma^2(\tau_{(1)})/L_f)}{K\hat{\eta}_\epsilon}  + 128\Delta^2(\hat{\tau}_\epsilon) + 256L_f\hat{\eta}_\epsilon \sigma^2(\hat{\tau}_\epsilon)\\
&\le \frac{16(F(x^0) -F_{\mathrm{low}} + \sigma^2(\tau_{(1)})/L_f)}{K\hat{\eta}_\epsilon} + 256L_f\hat{\eta}_\epsilon\sigma^2(\hat{\tau}_\epsilon) + \frac{\epsilon^2}{8}\\
&=\min_{\eta\in\big(0,\frac{1}{4L_f}\big]}\bigg\{\frac{16(F(x^0) -F_{\mathrm{low}} + \sigma^2(\tau_{(1)})/L_f)}{K\eta} + 256L_f \eta\sigma^2(\hat{\tau}_\epsilon)\bigg\} + \frac{\epsilon^2}{8}\\
&\overset{\eqref{ineq:tech-qm}}{\le}\frac{64L_f(F(x^0) -F_{\mathrm{low}} + \sigma^2(\tau_{(1)})/L_f)}{K} \\
&\qquad + 128\bigg[\frac{L_f(F(x^0) -F_{\mathrm{low}} + \sigma^2(\tau_{(1)})/L_f)\sigma^2(\hat{\tau}_\epsilon)}{K}\bigg]^{1/2} + \frac{\epsilon^2}{8},
\end{align*}
where the last relation is due to \eqref{ineq:tech-qm} and Lemma \ref{lem:tech-qm} with $(a,b,c)=(\frac{16(F(x^0) -F_{\mathrm{low}} + \sigma^2(\tau_{(1)})/L_f)}{K},256L_f\sigma^2(\hat\tau_\epsilon),\frac{1}{4L_f})$. Recall that $\iota_K$ is uniformly selected from $\{1,\ldots,K\}$. It then follows from this and the above relation that
\begin{align}
\E[\mathrm{dist}^2(0,\partial F(x^{\iota_K}))] & = \frac{1}{K}\sum_{k=0}^{K-1}\E[\mathrm{dist}^2(0,\partial F(x^{k+1}))]    \nonumber\\
&\le \frac{64L_f(F(x^0) -F_{\mathrm{low}} + \sigma^2(\tau_{(1)})/L_f)}{K} + 128\bigg[\frac{L_f(F(x^0) -F_{\mathrm{low}} + \sigma^2(\tau_{(1)})/L_f)\sigma^2(\hat{\tau}_\epsilon)}{K}\bigg]^{1/2}\nonumber \\
&\qquad + \frac{\epsilon^2}{8}.\nonumber
\end{align}
By this, one can observe that $\E[\mathrm{dist}^2(0,\partial F(x^{\iota_K}))]\le3\epsilon^2/8 + \epsilon^2/2 + \epsilon^2/8$ holds for all $K$ satisfying \eqref{K-lwbd-ncvx}, which completes the proof of this theorem.
\end{proof}

\bibliographystyle{abbrv}
\bibliography{ref}

\end{document}